%% file: main.tex
\documentclass[10pt,twocolumn,letterpaper]{article}

\usepackage[pagenumbers]{cvpr} %

\input{preamble}

\definecolor{cvprblue}{rgb}{0.21,0.49,0.74}
\usepackage[pagebackref,breaklinks,colorlinks,citecolor=cvprblue]{hyperref}

\def\paperID{***} %
\def\confName{CVPR}
\def\confYear{2024}

\title{From Correspondences to Pose:\\
Non-minimal Certifiably Optimal Relative Pose without Disambiguation}

\author{Javier Tirado-Garín \qquad Javier Civera\\
I3A, University of Zaragoza\\
{\tt\small \{j.tiradog,jcivera\}@unizar.es}
}

\input{commands}

\begin{document}
\crefname{appendix}{Supp.}{Supps.}
\Crefname{appendix}{Supplementary}{Supplementaries}
\maketitle
\input{sec/0_abstract}    
\input{sec/1_intro}

\input{sec/2_related_work}

\input{sec/3_method}

\input{sec/4_experiments}

\input{sec/5_conclusions}

{
    \small
    \bibliographystyle{ieeenat_fullname}
    \bibliography{main}
}

\input{sec/X_suppl}

\newpage

\end{document}

%% file: preamble.tex
\usepackage[dvipsnames]{xcolor}

\usepackage{amsthm}
\usepackage{mathtools}
\usepackage[percent]{overpic}
\usepackage{pifont}
\usepackage{subcaption}
\usepackage{lipsum}
\usepackage{gensymb}

\usepackage{booktabs}
\usepackage{tabularray}
\UseTblrLibrary{booktabs}
\usepackage{rotating}

\usepackage{tikz}
\usetikzlibrary{arrows.meta}

\usepackage{tabularray}
\usepackage[export]{adjustbox}
\usepackage{rotating}
\usepackage{makecell}

\usepackage{algorithm}
\usepackage{algpseudocode}

\usepackage{xspace}

%% file: commands.tex
\newcommand{\PAR}[1]{\vskip4pt \noindent{\bf #1~}}

\newtheorem{theorem}{Theorem}[section]

\newcommand{\bA}{\mathbf{A}}
\newcommand{\ba}{\mathbf{a}}
\newcommand{\bB}{\mathbf{B}}
\newcommand{\bbb}{\mathbf{b}}
\newcommand{\bC}{\mathbf{C}}
\newcommand{\bc}{\mathbf{c}}
\newcommand{\bD}{\mathbf{D}}
\newcommand{\bd}{\mathbf{d}}
\newcommand{\bE}{\mathbf{E}}
\newcommand{\be}{\mathbf{e}}
\newcommand{\bI}{\mathbf{I}}
\newcommand{\bL}{\mathbf{L}}
\newcommand{\bp}{\mathbf{p}}
\newcommand{\bQ}{\mathbf{Q}}
\newcommand{\bq}{\mathbf{q}}
\newcommand{\bR}{\mathbf{R}}

\newcommand{\bt}{\mathbf{t}}
\newcommand{\bV}{\mathbf{V}}
\newcommand{\bU}{\mathbf{U}}
\newcommand{\bv}{\mathbf{v}}
\newcommand{\bX}{\mathbf{X}}
\newcommand{\bx}{\mathbf{x}}

\newcommand{\SOthree}{\mathrm{SO}(3)}
\newcommand{\Stwo}{\mathcal{S}^2}
\newcommand{\ME}{\mathcal{M}_{\bE}}
\newcommand{\corr}{\mathbf{f}}
\newcommand{\bff}{\mathbf{f}}
\newcommand{\eps}{\varepsilon}

\newcommand{\epst}{\eps_{t}}
\newcommand{\bXEtqh}{\opt{\bX}_{\bE,\bt\bq,h}}
\newcommand{\ceqq}{\coloneqq}
\newcommand{\bXopt}{\mathbf{X}^{\star}}
\newcommand{\bXopte}{\bXopt_{\be}}
\newcommand{\bXoptt}{\bXopt_{\bt}}
\newcommand{\bXopttq}{\bXopt_{\bt\bq}}
\newcommand{\btq}{\mathbf{v}_{\bt\bq}}
\newcommand{\PSD}{\mathcal{S}_{+}}

\newcommand{\opt}[1]{#1^{\star}}
\newcommand{\bb}[1]{\mathbb{#1}}
\newcommand{\skeww}[1]{[#1]_{\times}}
\newcommand{\aug}[1]{\bar{#1}}
\newcommand{\twodots}{\mathinner{\ldotp\ldotp}}

\DeclareMathOperator{\rank}{rank}
\DeclareMathOperator{\diag}{diag}
\DeclareMathOperator{\trace}{tr}
\DeclareMathOperator{\Adj}{Adj}

\DeclareMathOperator{\vc}{vec}

\newcommand{\coloredarrow}[2]{%
  \tikz[baseline=-0.5ex]{\draw[#1, -Stealth] (0,0) -- (#2,0);}%
}

\definecolor{pastelred}{RGB}{235,119,119}
\definecolor{pastelorange}{RGB}{255,202,126}

\makeatletter
\newcommand{\setword}[2]{%
  \phantomsection
  #1\def\@currentlabel{\unexpanded{#1}}\label{#2}%
}
\makeatother

\newcommand{\crefprob}[1]{Prob. (\ref{#1})}

\providetoggle{foobool}

\newcommand{\method}[1]{{\small\textsf{#1}}}
\newcommand{\name}{\method{C2P}\xspace}
\newcommand{\nameFast}{\method{C2P-fast}\xspace}

%% file: sec/0_abstract.tex
\begin{abstract}
Estimating the relative camera pose from $n \geq 5$ correspondences between two calibrated views is a fundamental task in computer vision. This process typically involves two stages: 1) estimating the essential matrix between the views, and 2) disambiguating among the four candidate relative poses that satisfy the epipolar geometry. In this paper, we demonstrate a novel approach that, for the first time, bypasses the second stage. Specifically, we show that it is possible to directly estimate the correct relative camera pose from correspondences without needing a post-processing step to enforce the cheirality constraint on the correspondences.
Building on recent advances in certifiable non-minimal optimization, we frame the relative pose estimation as a Quadratically Constrained Quadratic Program (QCQP). By applying the appropriate constraints, we ensure the estimation of a camera pose that corresponds to a valid 3D geometry and that is globally optimal when certified. We validate our method through exhaustive synthetic and real-world experiments, confirming the efficacy, efficiency and accuracy of the proposed approach. \iftoggle{cvprfinal}{Code is available at \url{https://github.com/javrtg/C2P}.}{Our code can be found in the supp. material and will be released.}
\end{abstract}

%% file: sec/1_intro.tex
\section{Introduction}\label{sec:intro}

\input{figures/teaser/teaser}

Finding the relative pose between two calibrated views is crucial in many computer vision applications. This task is particularly relevant, among others, in Structure from Motion (SfM) \cite{Schonberger2016colmap,moulon2016openmvg}, and Simultaneous Localization And Mapping (SLAM) \cite{mur2015orb,mur2017orb2,campos2021orb3}. In SfM, it serves to geometrically verify the correspondences as well as to provide pairwise constraints for pose averaging schemes \cite{hartley2013rotation,lee2022hara}. In SLAM, besides correspondence verification, it is also used for bootstrapping the odometry of the camera and computing an initial estimate of the 3D map.

The relative pose problem has five observable degrees of freedom: three for the relative rotation between the cameras, and two for the direction of the relative translation. 
The standard approach for its computation \cite{Hartley2004mvg} begins by considering a set of $n$ pixel correspondences between the two images. These correspondences can be established through matching the descriptors of keypoints extracted from the images \cite{lowe2004distinctive,alcantarilla2013fast,detone2018superpoint,Lindenberger2023lightglue}, or more recently by estimating a (semi)dense 2D mapping between the views \cite{sun2021loftr,truong2023pdcnetpami,Edstedt2023dkm}. The pose is then computed by minimizing epipolar errors \cite{Hartley2004mvg}, requiring at least five correspondences. Solvers that handle $n=5$ correspondences are termed \emph{minimal} \cite{nister2004five, stewenius2006five} and those able to handle all the correspondences are called \emph{non-minimal} \cite{Briales_2018_CVPR, zhao2022nonmin}.

This paper focuses on non-minimal solvers. In a practical setup in which input correspondences may contain outliers, these solvers are essential within RANSAC \cite{fischler1981ransac,raguram2012usac} and Graduated Non-Convexity (GNC) \cite{yang2020gnc}. In RANSAC, non-minimal solvers are used to improve the accuracy of the \emph{so-far-best} and final models (initially computed with minimal solvers) thanks to noise cancellation of the inliers \cite{raguram2012usac}. In GNC, globally-optimal non-minimal solvers serve as fundamental building blocks to robustly solve a weighted instance of the problem in an iterative fashion.

Since the seminal paper by \citet{longuet1981computer} in 1981, computing this non-minimal estimate of the relative pose, involves two key steps: 
1) estimating the \emph{essential matrix}---which models the epipolar geometry of the problem---via an approximate (minimal or local optimization) method, or a certifiably globally-optimal method, and 2) disambiguating the true relative pose from those satisfying the same epipolar geometry. This ambiguity arises from ignoring the \emph{cheirality constraints} which enforce a valid 3D geometry (mainly that the 3D points observed in the image must be in front of the camera), 
resulting in a additional overhead that scales with the number of points.

In this paper, we demonstrate that this two-step approach, gold standard for more than 40 years, is not a must, and present a method that, for the first time directly estimates the relative pose without requiring posterior disambiguation.
For this purpose, we leverage convex optimization theory and demonstrate that we can obtain a certifiably globally-optimal solution that is geometrically valid. Additionally, our method can determine if the motion between the images is purely rotational. This information is relevant, since the translation vector is undefined under pure rotation and its estimation should not be trusted. Current methods require a posterior verification for the same purpose. We present a visual overview of both current approaches and our method in Figure \ref{fig:teaser}. 
Our main contributions are: 1) We provide a non-minimal certifiably globally-optimal method that, for the first time, solves the relative pose problem without the need of disambiguation. We derive the sufficient and necessary conditions to recover the optimal solution, 2) for this purpose, we also derive a novel characterization of the normalized essential manifold that is needed to enforce a geometrically valid solution, and 3) besides our theoretical contributions, we show experimentally that our method scales better than the alternatives in the literature.

%% file: figures/teaser/teaser.tex
\begin{figure}
    \centering
    \includegraphics[width=1.0\linewidth]{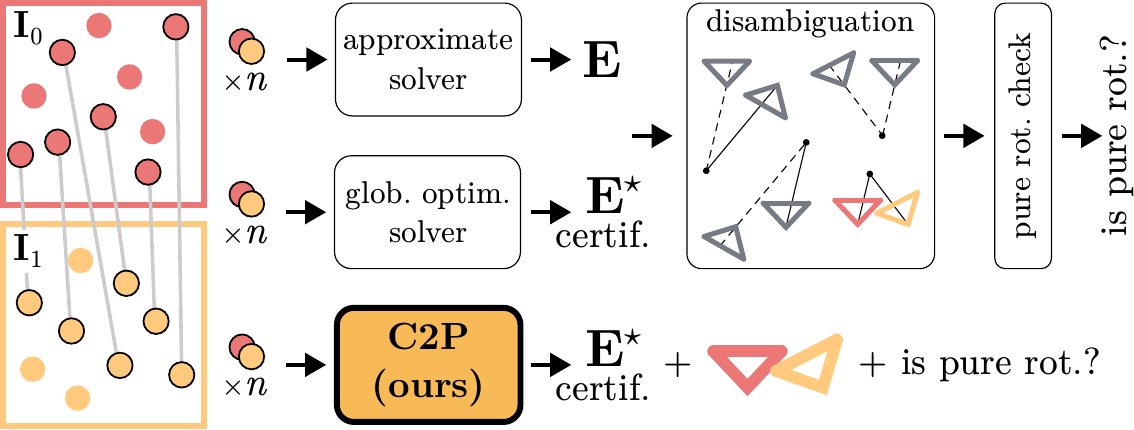}
    \caption{\textbf{Relative pose directly from matches, without posterior disambiguation and pure rotation checks.} Traditionally, estimating the relative pose involves two steps: 1) Estimating the essential matrix $\bE$ using an approximate or globally-optimal solver, and 2) disambiguating the unique geometrically valid pose among four candidate relative poses, with an additional step to determine if the motion is purely rotational. In this paper, we introduce \name, a globally-optimal and certifiable approach that, for the first time, solves the relative pose problem in a single step.}
    \label{fig:teaser}
\end{figure}

%% file: sec/2_related_work.tex
\section{Related Work}\label{sec:related_work}

\PAR{Non-minimal epipolar geometry.}
Initial methods for estimating the essential matrix (or \emph{fundamental} matrix in the uncalibrated case) \cite{longuet1981computer,hartley1997eight} rely on linear relaxations, which do not account for the nonlinearities arising from the constraints of the problem. As a result, these methods provide an approximate solution that needs to be projected onto the appropriate space.
While several methods \cite{Tron_2014_CVPR,tron2017essential,garciasalguero2021certifiable,helmke2007essential,garciasalguero2021fast} refine an approximate estimate through local optimization on the manifold of the essential matrix, they, despite being certifiable \cite{garciasalguero2021certifiable,garciasalguero2021fast}, cannot guarantee the global optimality of the solution.
To achieve global optimality, various methods \cite{hartley2009globalbnb, kneip2013directoptrotbnb, yang2014optimalbnb} employ Branch and Bound (BnB) techniques to explore the feasible parameter space and eliminate regions that are guaranteed not to contain the optimal solution. However, these methods can exhibit exponential time complexity in the worst case.
More recently, the use of relaxed Quadratically Constrained Quadratic Programs (QCQP) onto Semidefinite Programs (SDP) via the Shor's relaxation \cite{boyd2004convex,bao2011semidefinite}, has enabled methods to provide, and certify, globally-optimal solutions \cite{Briales_2018_CVPR,zhao2022nonmin,garciasalguero2022tighter,Karimian2023essential,Zhao2020generalizedessmat}. \citet{Briales_2018_CVPR} adopt an eigenvalue formulation of the problem \cite{kneip2013directoptrotbnb} and show that a tight relaxation of this non-convex formulation can be achieved through redundant constraints. \citet{zhao2022nonmin} achieves a significant increase in performance by optimizing with respect to a more efficient formulation, resulting in a reduced set of constraints and parameters. In addition, \citet{garciasalguero2022tighter} derive redundant constraints to improve the general tightness of the work by \citet{zhao2022nonmin}. Finally, \citet{Karimian2023essential} show that, under moderate levels of noise, a faster solution can be achieved using the Riemannian staircase algorithm \cite{burer2003lowrank,boumal2016staircase}.
In this paper, we address a common limitation of all previous approaches. We demonstrate that it is possible to avoid the posterior disambiguation of the four relative poses that satisfy their solutions. We achieve this also by relaxing a QCQP formulation of the problem. We design and introduce the constraints that take into account the 3D geometric meaning of the estimates, and provide a fast globally optimal solution for the geometrically correct relative pose.

\PAR{Related certifiably globally optimal methods.}
Similar relaxations to those used in previous methods are also applied in various closely related areas of computer vision. For example, applications of Shor's relaxation are found in tasks such as of solving Wahba's problem \cite{Yang2019quasarwahba}, pointcloud and 3D registration \cite{yang2021teaserplusplus,Briales2017registration} and multiview triangulation \cite{Harenstam2023triangulation}. Similarly, the Riemannian staircase algorithm, coupled with local manifold optimization, is used in tasks exploiting the low-rank nature of their solutions, including rotation averaging \cite{dellaert2020shonan} and pose synchronization \cite{rosen2019sesync,briales2017cartan}.

%% file: sec/3_method.tex
\section{Non-minimal solver for the relative pose}\label{sec:method}
In this work, we assume a central camera model, making our method suitable for, \eg, pinhole, fisheye and omnidirectional cameras. We consider that each correspondence, $i$, is parameterized by unit bearing vectors $\bff_{0,i}, \bff_{1,i}\in\Stwo$. 
Given a set of $n>5$ correspondences $\{\bff_{0,i}, \bff_{1,i}\}_{i=1}^n$, our goal is to directly estimate the relative pose, which we parameterize with a rotation matrix $\bR\in\SOthree$ and a unit-norm translation vector $\bt\in\Stwo$. These parameters define the manifold of normalized essential matrices \cite{helmke2007essential,zhao2022nonmin}:
\begin{equation}
    \ME \ceqq \{\bE \mid \bE = \skeww{\bt}\bR,~\bR\in\SOthree,~\bt\in\Stwo\}
    \text{\footnotemark}~.\label{eq:norm_ess_manif}
\end{equation}
\footnotetext{Where $\skeww{\bt}$ denotes the skew-symmetric matrix corresponding to $\bt$.}%
However, \cref{eq:norm_ess_manif} does not account for the geometric peculiarities and symmetry of the epipolar constraints \cite{Tron_2014_CVPR,tron2017essential}, which motivates us to investigate a necessary and sufficient characterization for directly estimating the relative pose. In \cref{sec:constraints}, we derive such characterization with a focus on making it amenable for a QCQP formulation (\cref{sec:qcqp}). We demonstrate that the relaxed QCQP yields a tight and certifiably globally optimal solution for the relative pose (\cref{sec:sdp}).

\subsection{Necessary and sufficient constraints}\label{sec:constraints}

\input{figures/fourposes/fourposes}

It is well known in the literature \cite[Sec. 9.6]{Hartley2004mvg} that four relative poses satisfy the same epipolar geometry. However, only one of these poses is geometrically valid, in the sense that it leads to an estimation of the 3D points being in front of the cameras\footnote{In practice, issues such as noise and small-scale translation relative to the observed scene can cause some 3D points to appear behind the cameras, despite a correct pose estimate. Consequently, cheirality is often verified for all points and then aggregated to robustly select the correct pose \cite{opencv_library,Schonberger2016colmap}.}. A common approach to circumvent this involves triangulating the points and checking for the positive sign of their depths \cite{Hartley2004mvg, nister2004five}, \emph{a posteriori}. Preliminary experiments showed that enforcing this positive-depth constraint \emph{during the optimization} is challenging and costly due to the complexities associated with the use of rotation matrices \cite{Briales_2018_CVPR,burer2023strengthened}, even after imposing convex hull constraints of $\SOthree$ \cite{sanyal2011orbitopes,saunderson2015semidefinite}. Consequently, we propose using simpler and more efficient constraints to disambiguate the pose during the optimization and visualize them in Figure \ref{fig:fourposes}.

\PAR{Rotation disambiguation.}
As shown in \cite{kneip2012finding}, regardless of the sign of the translation $\bt$, the two normals of the epipolar plane of a correspondence ($\bff_0, \bff_1$): $\bt\times\bff_0$ and $\bt\times\bR\bff_1$ satisfy:
\begin{equation}
    (\bt\times\bff_0)\cdot(\bt\times\bR\bff_1) > 0~,\label{eq:rot_dis1}
\end{equation}
Intuitively, considering $\bff_0$ and $\bR\bff_1$ in the same coordinate system, the smallest in-plane rotations required to align $\bff_0$ with $\bt$ and $\bR\bff_1$ with $\bt$ must be both clockwise or counter-clockwise. This condition is not met when the 3D point does not lie along one of the bearing vectors' \emph{beams}, which is the case for the (incorrect) reflected rotation matrix. 
However, besides involving $\bR$, \cref{eq:rot_dis1} is cubic on the unkowns, requiring a re-formulation with additional parameters to adapt it for a QCQP. A more straightforward solution arises upon realizing that $\bt\times\bR\bff_1=\skeww{\bt}\bR\bff_1=\bE\bff_1$, leading to:
\begin{align}
    (\bE\bff_1)\cdot(\skeww{\bt}\bff_0) = \bff_1^\top\bE^\top\skeww{\bt}\bff_0 > 0~,\label{eq:rot_dis2}
\end{align}
which now depends quadratically on the unknowns and is thus suitable for a QCQP formulation.

\PAR{Translation disambiguation.}
Building on the previous intuition that both $\bff_0$ and $\bR\bff_1$ require a (counter-)clockwise in-plane rotation to be aligned with $\bt$, it follows that:
\begin{equation}
    \bff_0^\top\bt - (\bR\bff_1)^\top\bt \geq 0~.\label{eq:dis_t1}
\end{equation}
If this condition is not met, the bearing vectors would not intersect along their beams, implying that the translation vector is the (incorrect) negative of $\bt$. The impact of this constraint on restricting the space of possible unit translations is shown in \cref{fig:disambiguation_t}. However, \cref{eq:dis_t1} again involves globally optimizing $\bR$. A more efficient approach is to optimize the rotated translation vector $\bq\ceqq\bR^\top\bt$, $\bq\in\Stwo$, resulting in:
\begin{equation}
    \bff_0^\top\bt - \bff_1^\top\bq \geq 0~.\label{eq:dis_t2}
\end{equation}
which is linear in the unknowns and can be incorporated in a QCQP with an homogenization variable \cite{giamou2023semidefinite,Briales_2018_CVPR}. However, this introduces the challenge of ensuring that $\bq=\bR^\top\bt$ still holds without optimizing $\bR$. We address this next by deriving a novel definition of the normalized essential manifold.

\input{figures/disambiguation_t/disambiguation_t}

\PAR{Manifold constraints.}
To ensure that $\bq=\bR^\top\bt$ holds during the optimization, we need to mutually constrain $\bt$ and $\bq$. Previous definitions of the normalized essential manifold, such as those involving the left, right and quintessential matrix sets \cite{zhao2022nonmin,garciasalguero2022tighter,Karimian2023essential}, do not include this kind of constraints. The most suitable constraints for our purposes are the ones involving the adjugate matrix of $\bE$: $\Adj(\bE)=\bq\bt^\top$, used in \cite{garciasalguero2022tighter} as redundant constraints. We show in Th. \ref{th:essmat} that these, along with norm constraints, suffice to define $\ME$, and to perform an efficient joint optimization of $\bE, \bt$ and $\bq$.

\input{proofs-theorems-algos/th_proof_essmat}

\settoggle{foobool}{true}
\iftoggle{foobool}{
\subsection{Recovery of the rotation}\label{sec:rot_recov}
Assuming tight solutions for $\bE\in\ME$ and $\bt,~\bq\in\Stwo$, we can directly recover the rotation $\bR\in\SOthree$ without disambiguation. A (normalized) essential matrix $\bE=\skeww{\bt}\bR$, depends linearly on $\bR$ and since $\rank(\bE)=2$, this provides six independent equations for solving $\bR$ (its nine elements). Thus, three additional independent equations are needed. Notably, $\bt$ lies in the nullspace of $\skeww{\bt}$, allowing us to find the remaining equations in the definition $\bq=\bR^\top\bt$. Hence, $\bR$ can be determined as the solution to this linear system. 
Since our method is empirically tight, our estimates $\bE, \bt$ and $\bq$ belong to their respective spaces, implying that the resulting $\bR$ belongs to $\SOthree$. 
While not theoretically necessary, for better numerical accuracy, we can: 1) project the resulting $\bR$ onto $\SOthree$ by classical means \cite{higham1989matrix}, and 2) consider the linearly dependent equations stemming from the equivalent definitions: $\bE=\skeww{\bt}\bR=\bR\skeww{\bq}$, $\bq=\bR^\top\bt$ and $\bt=\bR\bq$. The corresponding normal equations have as their LHS and RHS terms $2\bI_{9\times9}$ and $2\bt\bq^\top - \skeww{\bt}\bE - \bE\skeww{\bq}$, respectively (with the RHS expressed in vectorized form). Thus, $\bR$ can be solved in closed-form as
\begin{equation}
    \bR = \bt\bq^\top - \frac{1}{2} (\skeww{\bt}\bE + \bE\skeww{\bq})~.
\end{equation}
}
{\PAR{Recovery of the rotation.} Assuming that we have tight solutions for $\bE\in\ME$ and $\bt,~\bq\in\Stwo$, we can directly recover the rotation matrix $\bR\in\SOthree$ without requiring disambiguation. Recall that a (normalized) essential matrix satisfies $\bE=\skeww{\bt}\bR$, which depends linearly on $\bR$. Given that $\rank(\bE)=2$, this definition provides six linearly independent constraints for solving $\bR$ (its nine parameters). Thus, three additional independent equations are needed. Notably, $\bt$ lies in the nullspace of $\skeww{\bt}$, allowing us to find the remaining linear equations from the definition $\bq=\bR^\top\bt$. Consequently, the elements of $\bR$ can be computed as the solution of this linear system of equations. Since our solution is empirically tight, the estimates $\bE, \bt$ and $\bq$ belong to their respective spaces, implying that the solution $\bR$ belongs to $\SOthree$. While not theoretically necessary, for better numerical accuracy, we can project the resulting matrix $\bR$ to $\SOthree$ by classical means \cite{higham1989matrix}. The previous procedure holds even for pure rotational motions. In this situation, any pair of vectors $\bt$, $\bq$ related by $\bq = \bR^\top\bt$, leads to the same global minimum of the normalized epipolar errors (\cref{sec:qcqp}), but even in this case, the definitions $\bE=\skeww{\bt}\bR$ and $\bq = \bR^\top\bt$ still hold thanks to the constraints of \cref{th:essmat}. Thus, $\bR$ can be correctly recovered.}

\subsection{QCQP}\label{sec:qcqp}
The quadratic nature of our constraints motivates us to formulate the relative pose problem as a Quadratically Constrained Quadratic Program (QCQP). Since we parameterize the problem using the essential matrix $\bE$ (along with $\bt, \bq$), we draw inspiration from the efficient optimization strategy in \cite{zhao2022nonmin}. We optimize the same cost function. However, we differentiate from \cite{zhao2022nonmin} in that our approach automatically disambiguates the pose during the optimization, thanks to efficiently constraining the parameter space.

\PAR{Quadratic optimization}
In a noise-free scenario, an essential matrix $\bE$, satisfies the epipolar constraints: $\bff_{0,i}^\top\bE\bff_{1,i}=0$, for any correspondence $i$ \cite{Hartley2004mvg}. In practice, this does not happen and we instead optimize $\bE$ by minimizing the sum of squared normalized epipolar errors \cite{kneip2012finding,lee2020geometric, Muhle2022pnec}: 
\begin{equation}
    \min\sum_{i=1}^n (\bff_{0,i}^\top\bE\bff_{1,i})^2~.\label{eq:cost1}
\end{equation}
Since $\vc(\bff_{0,i}^\top\bE\bff_{1,i}) = (\bff_{0,i}\otimes\bff_{1,i})^\top\vc(\bE^\top)$ \cite{van2000ubiquitous}, this implies that \cref{eq:cost1} can be reformulated as:
\begin{alignat}{3}
    &\min\be^\top\bC\be~, \\
    &\be\ceqq\vc(\bE^\top),~
    \bC \ceqq \sum_i (\bff_{0,i}\otimes\bff_{1,i})(\bff_{0,i}\otimes\bff_{1,i})^\top~,\label{eq:matC}
\end{alignat}
where $\otimes$ is the Kronecker product \cite{van2000ubiquitous} and $\be\in\bb{R}^{9}$ represents the vector resulting from stacking the rows of $\bE$.

\PAR{Problem \setword{QCQP}{word:qcqp}.}
We formulate the relative pose problem as the following QCQP:
\begin{align}
    \min_{\be,\bt,\bq}\quad & \be^\top\bC\be~, \label{eq:qcqp_minterm}\\
    \textrm{s.t.} \quad 
    & \trace(\bE\bE^\top)=2,\quad\Adj(\bE)=\bq\bt^\top,~\label{eq:qcqp_essmat}\\
    & \bt^\top\bt=1,\quad\bq^\top\bq=1~, \label{eq:qcqp_essmat2} \\
    & \aug{\bff}_1\bE^\top\skeww{\bt}\aug{\bff}_0 - s_r^2 = 0~, \label{eq:qcqp_rotdis}\\
    & h\aug{\bff}_0^\top\bt - h\aug{\bff}_1^\top\bq - s_t^2 = 0~, \label{eq:qcqp_tdis} \\
    & h^2=1~\label{eq:qcqp_h}.
\end{align}
Both the minimization term and the constraints are quadratic. 
\cref{eq:qcqp_essmat,eq:qcqp_essmat2} correspond to the constraints presented in \cref{th:essmat}, and \cref{eq:qcqp_rotdis,eq:qcqp_tdis,eq:qcqp_h} to those presented at the beginning of \cref{sec:constraints}. Thus, we have $d=18$ parameters and $m=15$ constraints. As commented in \cref{sec:constraints}, these constraints are necessary and sufficient to disambiguate the relative pose during optimization. 

\PAR{Introduction of inequalities.} To transform the proposed inequalities of \cref{eq:rot_dis2,eq:dis_t2} into compact quadratic equalities that facilitate the use of off-the-shelf SDP solvers (\cref{sec:sdp}), we use two techniques: 1) we multiply \cref{eq:qcqp_tdis} with an \emph{homogenization variable} $h$, restricted to the values $\{-1,1\}$. This introduces a spurious (negative) solution if $h=-1$, but this is trivially checked and corrected \cite{Briales_2018_CVPR,giamou2023semidefinite}. 2) we introduce \emph{slack variables} \cite{fujisawa2002sdpa62} $s_r^2,s_t^2\in\bb{R}$. Since $s_r^2$ and $s_t^2$ are non-negative, this ensures the fulfillment of the inequalities. Additionally, $s_t^2$ offers an interesting advantage, as we show in \cref{sec:experiments}, since it enables the detection of pure rotational motions. Finally, note that we have adopted a modified notation for the bearings: $\aug{\bff}_i$, to denote that instead of selecting a random correspondence, we average at runtime the scalar coefficients from \cref{eq:qcqp_rotdis,eq:qcqp_tdis}, stemming from all the correspondences. This approach is motivated to average potential inlier noise in the correspondences. We detail this averaging in \Cref{app:supp_avg}.

\section{SDP relaxation and optimization}\label{sec:sdp}

Generally, optimizations of QCQPs like the one in \crefprob{word:qcqp} are (nonconvex) NP-hard. However, semidefinite programming (SDP) relaxations, have shown to be a powerful tool to tackle global optimization in computer vision \cite{cifuentes2022local}.  We draw inspiration from this, and relax the QCQP to an SDP. First, we write \crefprob{word:qcqp} in general form as:
\begin{align}
    \min_{\bx\in\bb{R}^d}\quad & \bx^\top\bC_0\bx~, \label{eq:qcqp_stdmin}\\
    \textrm{s.t.}\quad & \bx^\top \bA_i\bx = b_i, \quad i\in\{1, \twodots, m\} \label{eq:qcqp_stdconstraints}~,
\end{align}
where $\bx = [\be^\top, \bt^\top, \bq^\top, h, s_r, s_t]^\top$, $\bC_0$ is the adaptation of $\bC$ in \cref{eq:matC} to the formulation in \cref{eq:qcqp_stdmin}, \ie $\bC_0\in\PSD^{d}$\footnote{We denote as $\PSD^{d}$ the set of $d\times d$ positive semidefinite (PSD) matrices.} is a block-diagonal matrix whose only nonzero block is $\bC$, and $~\bA_i\in\mathcal{S}^{d}$ are symmetric matrices formed by the reformulation of \cref{eq:qcqp_essmat,eq:qcqp_essmat2,eq:qcqp_rotdis,eq:qcqp_tdis,eq:qcqp_h} to the form of \cref{eq:qcqp_stdconstraints}\footnote{In practice, off-the-shelf solvers like \textsf{SDPA} \cite{yamashita2010sdpa7} allow the use of a sparse representation of the constraints and only store the nonzero values of our constraints and minimization term.}.

To relax the QCQP, we use the cyclic property of the trace to realize that $\trace(\bx^\top\bC_0\bx)=\trace(\bC_0\bx\bx^\top)=\trace(\bC_0\bX)$, where $\bX\ceqq\bx\bx^\top\in\PSD^{d}$ constitutes a lifting of the parameters from $\bb{R}^d$ to $\PSD^{d}$. Doing similarly for \cref{eq:qcqp_stdconstraints}, allows us to define the following SDP:
\PAR{Problem \setword{SDP}{word:sdp}}
\begin{align}
    \min_{\bX \in \mathcal{S}^{d}}\quad &
    \trace(\bC_0\bX)~,\label{eq:sdp_min} \\
    \textrm{s.t.}\quad & 
    \trace(\bA_i\bX) = b_i, \quad i\in\{1, \twodots, m\} \label{eq:sdp_constraints}~,\\
    \quad & \bX\succeq0~.
\end{align}
SDPs are convex, being solvable globally-optimally in practice \cite{vandenberghe1996semidefinite}. The relaxation comes from not imposing any constraint on $\bX$ that ensures that the feasible set of \crefprob{word:sdp} matches that of \crefprob{word:qcqp}. 
If the global minimum of Probs. (\ref{word:sdp}) and (\ref{word:qcqp}) match, then we say that the relaxation is \emph{tight}. 
To obtain the solution estimate $\opt{\bx}$ we need to recover it from the globally optimal $\bXopt$. Interestingly, unlike other SDP relaxations in the literature \cite{Yang2019quasarwahba,Harenstam2023triangulation}, in the relative pose problem we cannot assume that tightness is achieved when $\rank(\bXopt)=1$, \ie, that $\opt{\bx}$ can be recovered by the factorization $\bXopt=\opt{\bx}(\opt{\bx})^\top$. Several works have noted this empirically \cite{Briales_2018_CVPR,zhao2022nonmin,Karimian2023essential}. For instance, the formulation of \cite{zhao2022nonmin} can be tight when $\rank(\bXopt)=2$. We prove this in \Cref{sec:supp_tightness}. Inspired by \cite{garciasalguero2022tighter}, to increase the overall tightness of our solutions,
we also include the following redundant quadratic constraints in our method:
\begin{equation}
    \bE\bE^\top = \skeww{\bt}\skeww{\bt}^\top~,\quad
    \bE^\top\bE = \skeww{\bq}\skeww{\bq}^\top~.\label{eq:redundant}
\end{equation}
We will refer to our method as \name, and as \nameFast when the redundant constraints of \Cref{eq:redundant} are not used.

\subsection{Relative pose recovery}\label{sec:recovery}
Given the optimal SDP solution, $\bXopt$, our goal is to recover the optimal relative pose. We define $\bXEtqh\ceqq\bXopt_{[1:16,1:16]}$ as the top-left $16\times16$ submatrix of $\bXopt$. Empirically, $\bXEtqh$ exhibits three nonzero singular values of varying magnitudes, while others are close to zero ($\sim10^{-6}$). This observation aligns with our parameterization, existing three linearly independent vectors that equally minimize \cref{eq:qcqp_minterm} (shown in \cref{th:qcqp}). However, only the vector corresponding to the largest singular value meets the proposed cheirality constraints (\cref{sec:constraints}). The reasoning for this is as follows: $\bXopt$ being a PSD matrix, can be decomposed (ignoring the singular values close to 0) as $\bXEtqh=\sum_i^3\sigma_i\bv_i\bv_i^\top$ when tight, and with $(\sigma_i,\bv_i)$ being a singular value-vector pair. This implies that 1) Each $\bv_i$ minimizes \cref{eq:qcqp_minterm} (see \cref{th:qcqp}), and 2) to satisfy the cheirality constraints, the singular vector $\bv_i$ corresponding to the correct solution, must have the biggest singular value in order to satisfy \cref{eq:qcqp_rotdis,eq:qcqp_tdis}. If we are strict, we should take into account the norms of $\bE$, $\bt$ and $\bq$ to see which singular pair contributes more positively in \cref{eq:qcqp_rotdis,eq:qcqp_tdis}. However, in practice, the solution is consistently found in the dominant singular vector (with biggest $\sigma_i$). We show some experiments corroborating this in Figure \ref{fig:svals}. Given the dominant vector $\bv_{0}$, we extract the parameters from its decomposition $\bv_{0}=[(\opt{\be})^\top, (\opt{\bt})^\top, (\opt{\bq})^\top,h]^\top$. We then appropriately scale them, and follow the rotation recovery approach of \cref{sec:rot_recov}. The complete method is outlined in \cref{alg:c2p}. For more details, please refer to \cref{sec:supp_details}.

\input{proofs-theorems-algos/th_proof_tight_cond_sdr_QCQP}

\input{figures/singular_values/singular_vals_ratios}

\input{proofs-theorems-algos/algorithm}

%% file: figures/fourposes/fourposes.tex
\begin{figure}
    \newcommand{\cmark}{\ding{51}}%
    \newcommand{\xmark}{\ding{55}}%
    \centering
    \begin{overpic}[width=1.0\linewidth]{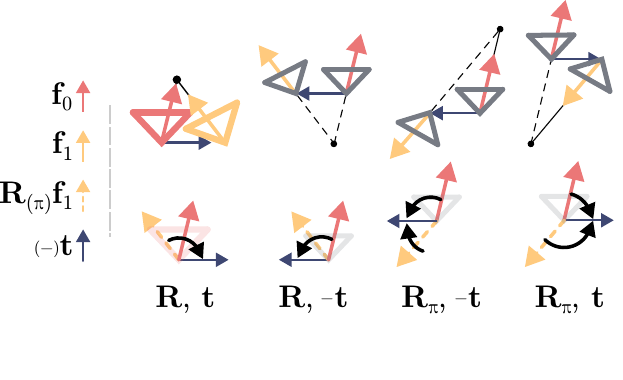}
        \put (22,6) {\cref{eq:rot_dis2} \textcolor{green}{\cmark}}
        \put (22,0) {\cref{eq:dis_t2} \textcolor{green}{\cmark}}
        \put (42,6) {\cref{eq:rot_dis2} \textcolor{green}{\cmark}}
        \put (42,0) {\cref{eq:dis_t2} \textcolor{red}{\xmark}}
        \put (62,6) {\cref{eq:rot_dis2} \textcolor{red}{\xmark}}
        \put (62,0) {\cref{eq:dis_t2} \textcolor{red}{\xmark}}
        \put (83,6) {\cref{eq:rot_dis2} \textcolor{red}{\xmark}}
        \put (83,0) {\cref{eq:dis_t2} \textcolor{green}{\cmark}}
    \end{overpic}
    \caption{\textbf{Necessary geometric conditions.} When removing the rotational flow  between the bearing vectors \cite{kneip2012finding}, \ie considering $\bff_0$ and $\bR\bff_1$, they must exhibit the same (counter-)clockwise turn w.r.t. the translation (\cref{eq:rot_dis2}). Otherwise, the rotation must be a reflected version, $\bR_{\pi}$, of the true rotation, $\bR$. Considering the correct rotation, $\bff_0$ must have greater projection onto the translation than $\bR\bff_1$ (\cref{eq:dis_t2}). Otherwise, the bearings would not meet along the direction of their beams (\protect\coloredarrow{pastelred}{3mm}, \protect\coloredarrow{pastelorange}{3mm}), implying that the translation is flipped ($-\bt$) w.r.t. the correct one, $\bt$. Therefore, besides avoiding triangulation, these constraints completely disambiguate the relative pose and are generally applicable to central camera models since they do not rely on traditional positive-depth constraints.}
    \label{fig:fourposes}
\end{figure}

%% file: figures/disambiguation_t/disambiguation_t.tex
\begin{figure*}
    \centering
    \includegraphics[width=.85\textwidth]{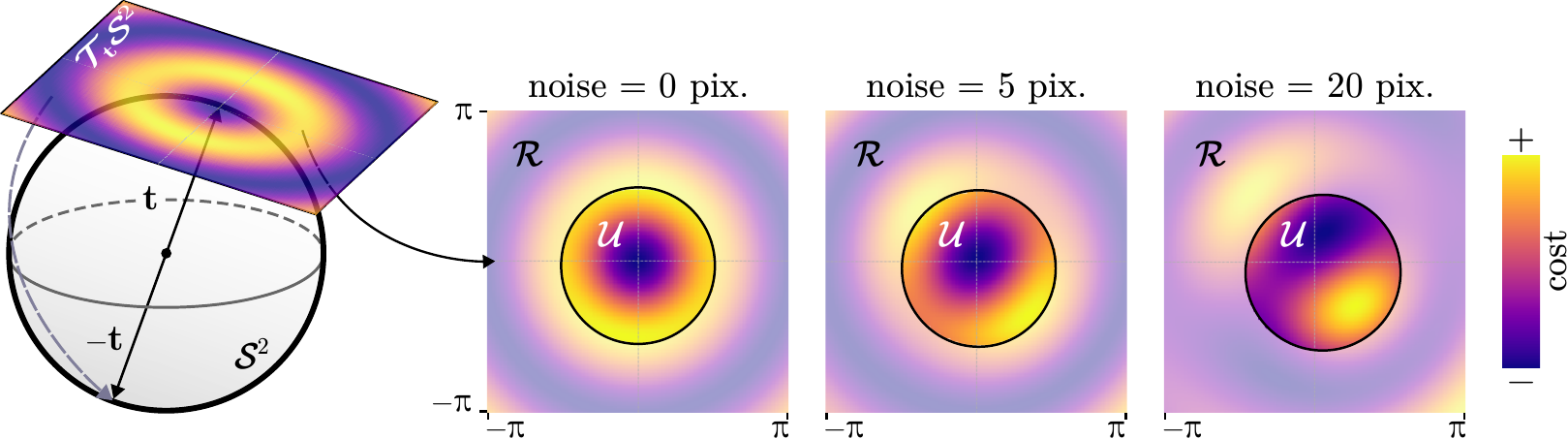}
    \caption{\textbf{Automatic disambiguation of the relative pose.} Our method restricts the set of possible rotations and translations (unit vectors, due to scale ambiguity) for solving the relative pose, by incorporating cheirality constraints in the optimization. We visualize this for the translation with \emph{cost} maps of squared epipolar errors in the tangent space at the ground-truth translation, $\mathcal{T}_{\bt}\Stwo$, and for different levels of noise. Elements in $\mathcal{T}_{\bt}\Stwo$ are mapped to the sphere along geodesics using the exponential map, which is a bijective mapping for $\lVert\bv\rVert\leq\pi$ with $\bv\in\mathcal{T}_{\bt}\Stwo$ \cite{boumal2023intromanifolds}. This enables us to show, on the right, the space not satisfying the constraint of \cref{eq:dis_t2} with lower opacity, named $\mathcal{R}$. As can be seen, the global minimum corresponding to $\bt$, always lies within the unrestricted space, named $\mathcal{U}$, while it excludes the global minima corresponding to $-\bt$. Therefore the solver is able to automatically select the translation with the correct sign as the solution.}
    \label{fig:disambiguation_t}
\end{figure*}

%% file: proofs-theorems-algos/th_proof_essmat.tex
\begin{theorem}\label{th:essmat}
    A real $3\times3$ matrix, $\bE$, is an element of $\mathcal{M}_{\bE}$ if and only if it satisfies:
    \begin{equation}
        \text{(i)} \quad \trace(\bE\bE^\top)=2 
        \quad \mathrm{and} \quad
        \text{(ii)} \quad \Adj(\bE) = \bq \bt^\top~,
    \end{equation}
    for two vectors $\bt,\bq\in\Stwo$ 
    and where $\Adj(\bE)$ represents the \href{https://en.wikipedia.org/wiki/Adjugate_matrix}{adjugate matrix} \cite{horn2012matrix} of $\bE$.
\end{theorem}

\begin{proof}
    For the \emph{if} direction, assume $\Adj(\bE) = \bq \bt^\top$ and $\trace(\bE\bE^\top)=2$. The outer product $\bq\bt^\top$ is a rank-1 matrix, implying that $\rank(\bE)=2$. Thus, the singular value decomposition (SVD) of $\bE$ is
    $\bE = \bU \bD \bV^\top$, with $\bD \ceqq \diag(\sigma_0,~\sigma_1,~0)$, $\sigma_0, \sigma_1\in\bb{R}^{+}$ and $\bU, \bV \in \mathrm{O}(3)$.
    The adjugate of $\bE$ can then be expressed as\footnote{We use the adjugate matrix property $\Adj(\bA\bB)=\Adj(\bB)\Adj(\bA)$, for any $\bA,\bB\in\bb{R}^{n\times n}$. We also use that $\Adj(\bQ)=\det(\bQ)\bQ^\top=(\pm1)\bQ^\top$ for any orthogonal matrix $\bQ\in\mathrm{O}(n)$.}:
    \begin{align}
        \Adj(\bE) &= \Adj(\bU \bD \bV^\top)~, \\
        &= \Adj(\bV^\top) \Adj(\bD) \Adj(\bU)~, \\
        &= \pm \bV \diag(0,~0,~\sigma_0\sigma_1) \bU^\top~, \\
        &= \pm \sigma_0\sigma_1 \mathbf{v}_2 \mathbf{u}_2^\top~,
    \end{align}
    where $\mathbf{v}_2$ and $\mathbf{u}_2$ are the third columns of $\bV$ and $\bU$, respectively. Since the vectors involved in both outer products, $\mathbf{v}_2 \mathbf{u}_2^\top$ and $\bq \bt^\top$, are unit vectors, this implies that $\sigma_0\sigma_1=1$. Furthermore, since $\trace({\bE\bE^\top})=2$, it follows\footnote{
    Since $\bE\bE^\top=\bU\bD^2\bU^\top$, it follows that $\trace(\bE\bE^\top)=\trace(\bU\bD^2\bU^\top)=\trace(\bD^2\bU^\top\bU)=\trace(\bD^2)=\sigma_0^2+\sigma_1^2$.} that $\sigma^2_0+\sigma^2_1=2$. These two equations lead to the biquadratic equations:
    \begin{equation}
        \sigma_i^4-2\sigma_i^2+1 
        = (\sigma_i-1)^2 (\sigma_i+1)^2 = 0, 
        \quad i\in\{0,1\}~, 
    \end{equation}
    which have $\sigma_i=1$ as positive roots, meaning that $\bE$ has two non-zero singular values, both equal to $1$. Hence, $\bE$ is an essential matrix that belongs to $\mathcal{M}_{\bE}$.

    For the \emph{only if} direction, assume $\mathbf{E}$ is an essential matrix in $\ME$. Thus, condition (i) is satisfied, and there exist two unit vectors $\bq,\bt\in\Stwo$ and a rotation matrix $\bR\in\SOthree$, satisfying that $\bE = \skeww{\bt}\bR = \bR\skeww{\bq}$ and $\bq = \bR^\top\bt$. We can show then that
    \begin{align}
        \Adj(\bE) &= \Adj(\skeww{\bt}\bR)~, \\
        &= \Adj(\bR)\Adj(\skeww{\bt})~, \\
        &= \bR^\top \bt\bt^\top 
        = \bq\bt^\top~.
    \end{align}
    Thus, condition (ii) is also satisfied.
\end{proof}

%% file: proofs-theorems-algos/th_proof_tight_cond_sdr_QCQP.tex
\PAR{Sufficient and necessary condition for global optimality.}
To alleviate the notation in the following proof, let us define the following auxiliary variables:
\begin{alignat}{3}
    \bXopte&\ceqq\bXopt_{[1:9,1:9]}~,  \quad \bXopttq&&\ceqq\bXopt_{[10:15,10:15]}~,\\
    \bXEtqh&\ceqq\bXopt_{[1:16,1:16]}~,\quad \btq&&\ceqq[\bt^\top, \bq^\top]^\top,
\end{alignat}
where $\bXopt_{[i:j,i:j]}$ represents the submatrix of $\bXopt$ extracted by selecting the rows and columns ranging from index $i$ through index $j$. 

\begin{theorem}\label{th:qcqp}
    The semidefinite relaxation of \crefprob{word:qcqp} is tight if and only if 
    $\rank(\bXEtqh)\in[1,3]$ and
    its submatrices $\bXopte, \bXopttq$ are rank-1.
\end{theorem}
\begin{proof}
    For the \emph{only if} part, assume the relaxation is tight. Then, $\bXopt$ is contained in the convex hull of the linearly independent rank-1 solutions to the relative pose problem \cite{Briales_2018_CVPR}. 
    We can form at most three feasible linearly-independent solutions that minimize the cost equally\footnote{Recall that the correct solution, besides minimizing the cost term, satisfies the cheirality constraints, and is the singular vector associated with the highest singular value, as commented in \cref{sec:recovery}.}, \eg:
    \begin{equation}
        \bx_0\ceqq
        \begin{bmatrix}
            \be \\
            \btq \\
            h
        \end{bmatrix},~
        \bx_1\ceqq
        \begin{bmatrix}
            \be \\
            -\btq \\
            h
        \end{bmatrix},~
        \bx_2\ceqq
        \begin{bmatrix}
            \be \\
            -\btq \\
            -h
        \end{bmatrix}~,
        \label{eq:linindep_sols}
    \end{equation}
    Note that $\bt$ and $\bq$ must share the same sign to be feasible. Thus, the convex combination ($\alpha_i\in\bb{R}^{+}, \sum_i\alpha_i=1,~i\in\{0, 1, 2\}$) that forms the globally optimal matrix can be expressed as:
    \begin{align}
        \bXEtqh&=\alpha_0\bx_0\bx_0^\top
        + \alpha_1\bx_1\bx_1^\top
        + \alpha_2\bx_2\bx_2^\top, \\
        &=\begin{bmatrix}
            \be\be^\top 
            & a_0\be\btq^\top 
            & a_1 h \be \\
            a_0\btq\be^\top 
            & \btq\btq^\top
            & a_2 h \btq \\
            a_1 h \be^\top
            & a_2 h \btq^\top
            & 1
        \end{bmatrix},
    \end{align}
    where $a_0 \ceqq (\alpha_0-\alpha_1-\alpha_2)$, $a_1 \ceqq (\alpha_0+\alpha_1-\alpha_2)$, $a_2 \ceqq (\alpha_0-\alpha_1+\alpha_2)$.
    Therefore, $\rank(\bXEtqh)\in[1,3]$ and $\rank(\bXopte)=\rank(\bXopttq)=1$.
    
    For the \emph{if} part, we build upon \cite[Theorem 2]{zhao2022nonmin}. Since $\bXopt$ is a positive semidefinite (PSD) matrix, $\bXopte$ and $\bXopttq$ are also PSD as they are principal submatrices of $\bXopt$ \cite{strang1980its}. Given that $\bXopte$ and $\bXopttq$ are both rank-1 matrices, it follows that there exist two vectors $\opt{\be}\in\bb{R}^9$ and $\opt{\btq}\in\bb{R}^6$ that fulfill the primal problem's constraints and satisfy $\opt{\be}(\opt{\be})^\top=\bXopte$ and $\opt{\btq}(\opt{\btq})^\top=\bXopttq$.
    
    Regarding the rank of $\bXEtqh$, since it is PSD, it can be factorized as $\bXEtqh=\bL\bL^\top$, where $\bL\in\bb{R}^{16\times r}$ and $r\ceqq\rank(\bXEtqh)$. Thus, to satisfy the rank-1 property of $\bXopte$ and $\bXopttq$, each column $k$ of $\bL$ must be given by: $[a_k\be^\top,b_k\btq^\top, c_k h]^\top$, for some scalars $a_k,b_k,c_k\in\bb{R}$. This constraint limits the rank of $\bXEtqh$ to at most $3$, as any additional column in $\bL$ would be a linear combination of the existing ones. Therefore, since $\bXEtqh$ must be feasible, this implies that $\rank(\bXEtqh)\in[1,3]$ and that it is a convex combination of the three linearly independent solutions, stemming from our parameterization of the problem, and thus the relaxation is tight.
\end{proof}

%% file: figures/singular_values/singular_vals_ratios.tex
\begin{figure}
    \centering
    \includegraphics[width=0.75\linewidth]{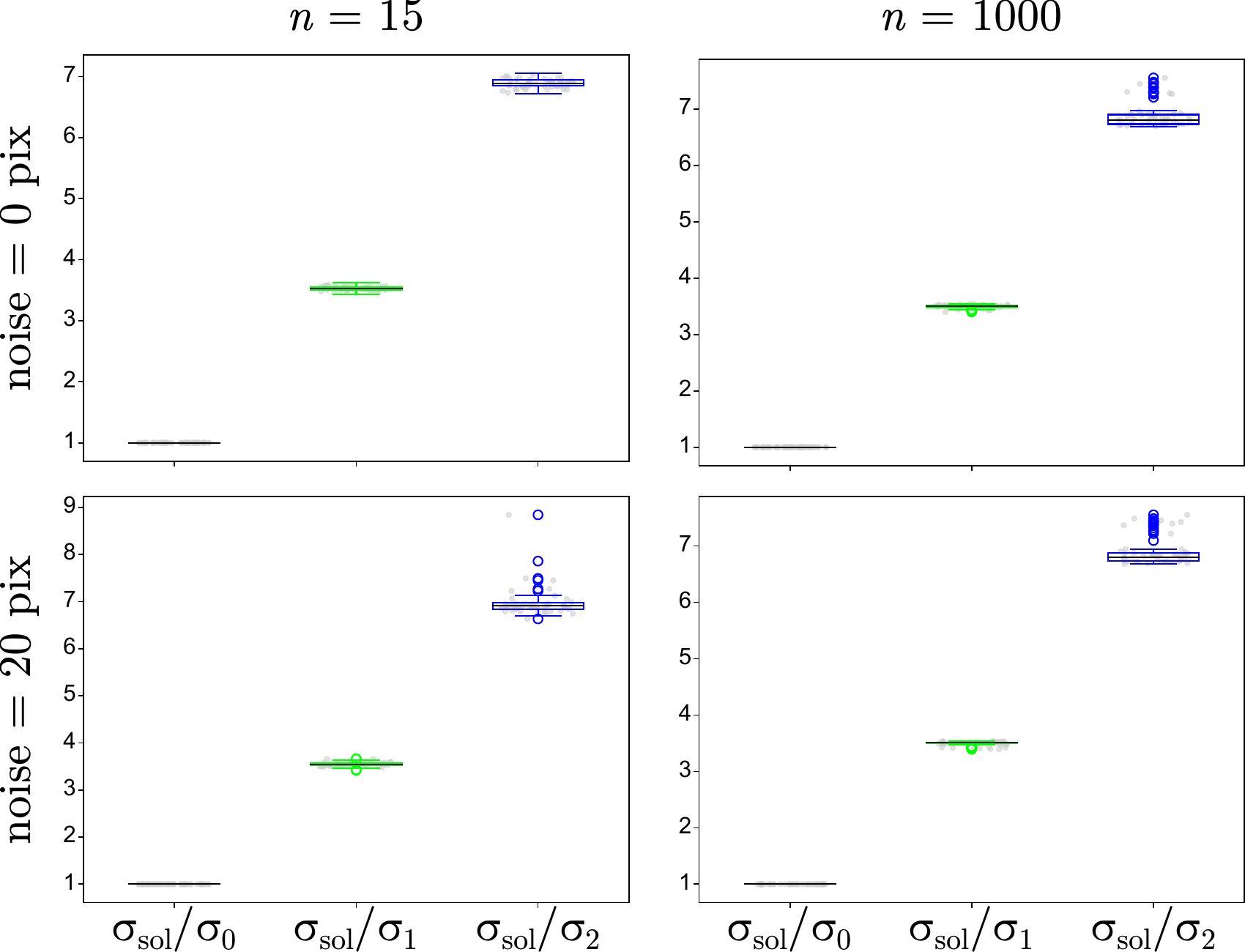}
    \caption{\textbf{The solution is found in the dominant singular vector.} We show, across different levels of noise and number of correspondences (we repeat each experiment 100 times), boxplots for the ratios $\sigma_{\text{sol}}/\sigma_0$ (left box), $\sigma_{\text{sol}}/\sigma_1$ (middle box), $\sigma_{\text{sol}}/\sigma_2$ (right box). $\sigma_{\text{sol}}$ corresponds to the  singular value whose vector contains the estimate closest to the ground-truth. $\sigma_0, \sigma_1, \sigma_2$ represent, in order ($\sigma_0>\sigma_1>\sigma_2$) the top-three singular values (the rest are close to zero). As can be seen, the solution vector consistently corresponds to the dominant singular vector \ie $\sigma_{\text{sol}}/\sigma_0=1$. Therefore, 
    we can directly select the dominant vector to recover the solution.}
    \label{fig:svals}
\end{figure}

%% file: proofs-theorems-algos/algorithm.tex
\begin{algorithm}
\algrenewcommand\algorithmicrequire{\textbf{Input:}}
\algrenewcommand\algorithmicensure{\textbf{Output:}}
\newcommand{\mycomment}[1]{\textcolor{gray}{\# #1}}
\newcommand{\myvar}[1]{{\tt #1}}
\caption{\name: Relative pose without disambiguation}\label{alg:c2p}
\begin{algorithmic}[1]
    \Require List of correspondences $\{\corr_{0,i}, \corr_{1,i}\}_{i=1}^{n}$, threshold $\epst\in\bb{R}^{+}$, optional weights $\{w_i\}_{i=1}^{n}$
    \Ensure $\opt{\bE}\in\ME$, geometrically correct relative pose $(\opt{\bR}\in\SOthree,\opt{\bt}\in\Stwo)$, \myvar{certif}, \myvar{is\_pure\_rot}
    \State \mycomment{Precomputed constraint elements in \cref{eq:sdp_constraints} form}
    \State $\{(\bA_i, b_i)\}_{i=1}^{m-2} \gets $ \cref{eq:qcqp_essmat,eq:qcqp_essmat2,eq:qcqp_h,eq:redundant}
    \State \mycomment{Cost matrix and data-dependent constraint matrices}
    \State $\bC, \bA_{r}, \bA_{t} \gets$ \cref{eq:matC,eq:qcqp_rotdis,eq:qcqp_tdis}
    \State \mycomment{Solve SDP and recover geometrically-valid rel. pose}
    \State $\bXopt \gets \Call{SDP}{\bC, \{(\bA_i, \bbb_i)\}_{i=1}^{m-2}, \bA_{r}, \bA_{t}}$ 
    \hfill\mycomment{\eg \cite{yamashita2010sdpa7}}
    \State $\opt{\bE},~\opt{\bt},~\opt{\bq},~s_t \gets \Call{Eig}{\bXopt}$ 
    \hfill\mycomment{\cref{sec:recovery,sec:supp_scaling}}
    \State \mycomment{Global-optimality and (near) pure-rotation certificates}
    \State \myvar{certif} $\gets \Call{Bool}{\bXopt\text{ meets Th. \ref{th:qcqp} rank conditions}}$
    \State \myvar{is\_pure\_rot} $\gets \Call{Bool}{s_t^2 < \epst}$
    \State \mycomment{Improve numerical accuracy---if needed (\cref{sec:supp_num_acc})}
    \If{$s_t^2 < 10^{-4}$}
        \State $\opt{\bE},~\opt{\bt},~\opt{\bq} \gets \Call{SignedEig}{\bXopt_{[:15,:15]}}$
        \hfill\mycomment{\cref{sec:supp_num_acc}}
    \EndIf
    \State $\opt{\bR} \gets \Call{RecoverRotation}{\opt{\bE},\opt{\bt},\opt{\bq}}$ 
    \hfill\mycomment{\cref{sec:rot_recov}}
    \State \Return $\opt{\bE},~\opt{\bR},~\opt{\bt}$, \myvar{certif}, \myvar{is\_pure\_rot}
\end{algorithmic}
\end{algorithm}

%% file: sec/4_experiments.tex
\section{Experiments}\label{sec:experiments}

\input{figures/runtime/runtime_fig}

\input{figures/acc_npoints/acc_npoints}

To solve \crefprob{word:sdp}, we use \method{SDPA} \cite{fujisawa2002sdpa62,yamashita2010sdpa7}, a solver for general SDPs based on primal-dual interior-point methods. For the experiments we use an Intel Core i5 CPU at 3.00 GHz. 
For comparison with non-minimal works, we consider the globally optimal methods of \cite{zhao2022nonmin,garciasalguero2022tighter}.

\PAR{Synthetic data.} Following \cite{Briales_2018_CVPR,zhao2022nonmin,garciasalguero2022tighter,Karimian2023essential} we test our method with synthetic experiments. To simulate the scenes and cameras, we follow the procedure of \cite{zhao2022nonmin}. Specifically, we set the absolute pose of one camera to the identity. For the other camera, the direction of its relative translation is uniformly sampled with a maximum magnitude of 2, and the relative rotation is generated with random Euler angles bounded to
0.5 radians in absolute value. This generates random relative poses as they would appear in practical situations. We uniformly sample point correspondences around the origin with a distance varying between 4 and 8 and then transform each to the reference frames of the cameras to obtain the unit bearing vectors. We add noise to the bearings assuming a spherical camera \ie we extract the tangent plane of each bearing and add uniformly distributed random noise expressed in pixels inside this plane.

\PAR{Real data.} Following \cite{zhao2022nonmin,garciasalguero2022tighter}, we additionally consider the six sequences from the dataset \cite{strecha2008benchmark}. We generate 97 wide-baseline image pairs by grouping adjacent images. For each image pair, we extract correspondences with SIFT features \cite{lowe2004distinctive}. Then, we use RANSAC to filter out wrong correspondences. These results are consistent with the synthetic scenarios and we provide them in \Cref{sec:supp_exp} due to space limitations.

\input{figures/pure_rot/pure_rot}

\PAR{Execution time}
Although our main contributions are theoretical, our method has the advantage of scaling better with the number of correspondences, thanks to not requiring posterior disambiguation. We compare our runtimes with those of \cite{garciasalguero2022tighter,zhao2022nonmin} in \cref{fig:runtime} for varying number of points. Since \cite{garciasalguero2022tighter,zhao2022nonmin} need to disambiguate the four candidate poses satisfying the same epipolar geometry, we consider two techniques: (\textbf{T}) classic triangulation of the correspondences, followed by a positive-depth check \cite{Hartley2004mvg} (we use OpenCV's {\small \texttt{recoverPose}} \cite{opencv_library} for this) and, for fairer comparison, we also consider a method (\textbf{M}) that avoids triangulation and is based on checking the estimated sign of the norms computed with the midpoint method \cite{beardsley1994navigation,lee2019bmvctriang}. Specifically, we select the camera pose that satisfies the most:
\begin{equation}
    \resizebox{0.88\linewidth}{!}{%
        ${\displaystyle
            (\bR\bff_1\times\bff_0)
            \cdot
            (\bff_0\times\bt)>0~,\quad
            (\bR\bff_1\times\bff_0)
            \cdot
            (\bR\bff_1\times\bt)>0~,%
        }$}%
        \label{eq:midpoints}
\end{equation}
for all correspondences. A geometric derivation of \cref{eq:midpoints} is found in \cite{lee2019bmvctriang}. In \Cref{sec:supp_deriv_midpoints} we provide an algebraic derivation. As can be seen, \nameFast and \cite{zhao2022nonmin} + M, are the fastest methods when the number of correspondences is low. Additionally, there exist small difference ($<2$ ms) w.r.t. \name which includes redundant constraints. 
However, for $n>10^{3}$ (common in dense matchers \cite{truong2023pdcnetpami, Edstedt2023dkm, edstedt2023roma}) the disambiguation starts dominating the runtime of \cite{zhao2022nonmin} + M, while both \name and \nameFast present better scaling.

\PAR{Accuracy vs number of correspondences}
We first test the accuracy of the methods w.r.t. the number of correspondences ($n$). To better visualize their behavior, we consider two regimes: R1) $n\in[12, 30]$ and R2) $n\in[10^2,10^4]$). We fix the noise to $1$ pixel (the same conclusions hold at different noise levels, as shown in \cref{sec:supp_exp}). For both regimes we repeat the experiments 1000 times. We set the step size of $n$ to 1 in R1 and 400 in R2. In \cref{fig:acc_vs_n}, we report the mean rotation: $\arccos(0.5(\trace(\bR_{\text{true}}^\top,\bR)-1))$ and translation errors $\arccos(\bt_{\textrm{true}}^\top\bt)$ in degrees, where $\bR_{\text{true}}, \bt_{\textrm{true}}$ is the ground-truth pose. From this experiment, we conclude that redundant constraints help to improve the accuracy when $n$ is low, as in R1), both \name and \cite{garciasalguero2022tighter}, outperform \cite{zhao2022nonmin}, while our \name is faster (\cref{fig:runtime}). In R2) all methods are equally accurate, while \name and \nameFast start becoming the fastest methods. In practice, we can easily switch between \name and \nameFast based on $n$, thus achieving a good balance in speed and accuracy when compared to the alternatives.

\PAR{Pure rotations}
In this experiment, we verify our method's efficacy under near-pure rotational motions, which is known to be challenging \cite{Briales_2018_CVPR,kneip2012finding}. Besides, the slack variable $s_t^2$, corresponding to \cref{eq:qcqp_tdis}, enables the detection of such motions through simple thresholding. Intuitively, \cref{eq:qcqp_tdis} corresponds to \cref{eq:dis_t1} and this inequality becomes $0$ under pure rotations ($\bff_0=\bR\bff_1$ in this case). In this experiment, we vary the translation magnitude and do 1000 repetitions for each, setting the noise to $0.5$ pixels. Results in \Cref{sec:supp_exp}, show that the rotation accuracy is unaffected by the translation magnitude, but as this decreases, the estimate $\bt$ worsens since the minimization of the epipolar errors become more insensitive to $\bt$. \name behaves similarly as \cite{garciasalguero2022tighter,zhao2022nonmin}, while \name has the advantage of directly identifying near-pure rotations by just checking $s_t^2$, thus avoiding extra steps such as the required in \cite{zhao2022nonmin} (see \cref{fig:pure_rot}).

%% file: figures/runtime/runtime_fig.tex
\begin{figure}[!h]
    \centering
    \includegraphics[width=0.9\linewidth]{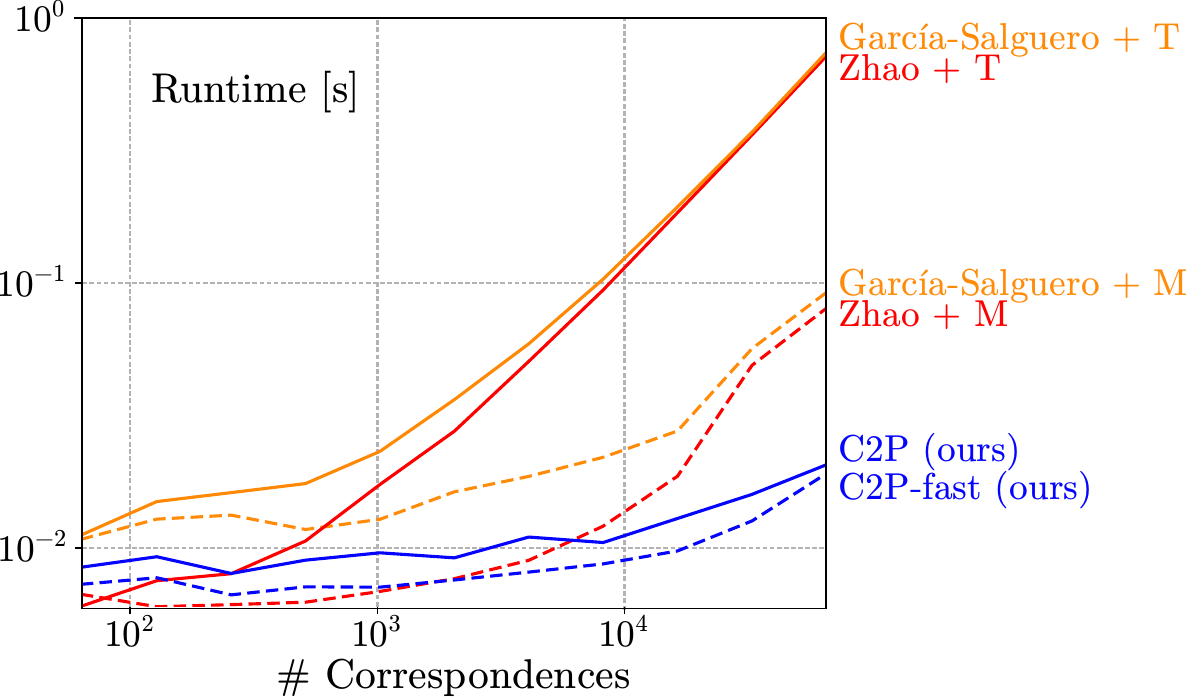}
    \caption{\textbf{Run time vs number of correspondences.} We compare the execution time (in sec.) of \name against \citet{zhao2022nonmin} and \citet{garciasalguero2022tighter}. Unlike \name, \cite{zhao2022nonmin, garciasalguero2022tighter} need a post-processing step to disambiguate the four valid candidate poses. For this, we use two methods: (\textbf{T}) the classic cheirality check \cite{Hartley2004mvg}, 
    that triangulates the points and checks for positive-depths,
    and (\textbf{M}) A faster alternative, that avoids triangulation and instead checks \cref{eq:midpoints}.
    \nameFast and \cite{zhao2022nonmin} + M, are the fastest when the number of correspondences is low, and there exist small difference ($<2$ ms) if we use redundant constraints (\name). 
    However, for $n>10^{3}$ (common in dense matchers \cite{truong2023pdcnetpami, Edstedt2023dkm, edstedt2023roma}) the disambiguation step starts dominating the runtime of \cite{zhao2022nonmin} + M, while both versions of our method (\name and \nameFast) present up to 4x and 35x times better runtimes w.r.t. the fastest, and slowest alternative, respectively.}
    \label{fig:runtime}
\end{figure}

%% file: figures/acc_npoints/acc_npoints.tex
\begin{figure*}[!ht]
    \centering
    \begin{tblr}{
        colspec={cc},
        colsep=10pt,
        rowsep=0pt,
    }
        \includegraphics[width=8cm]{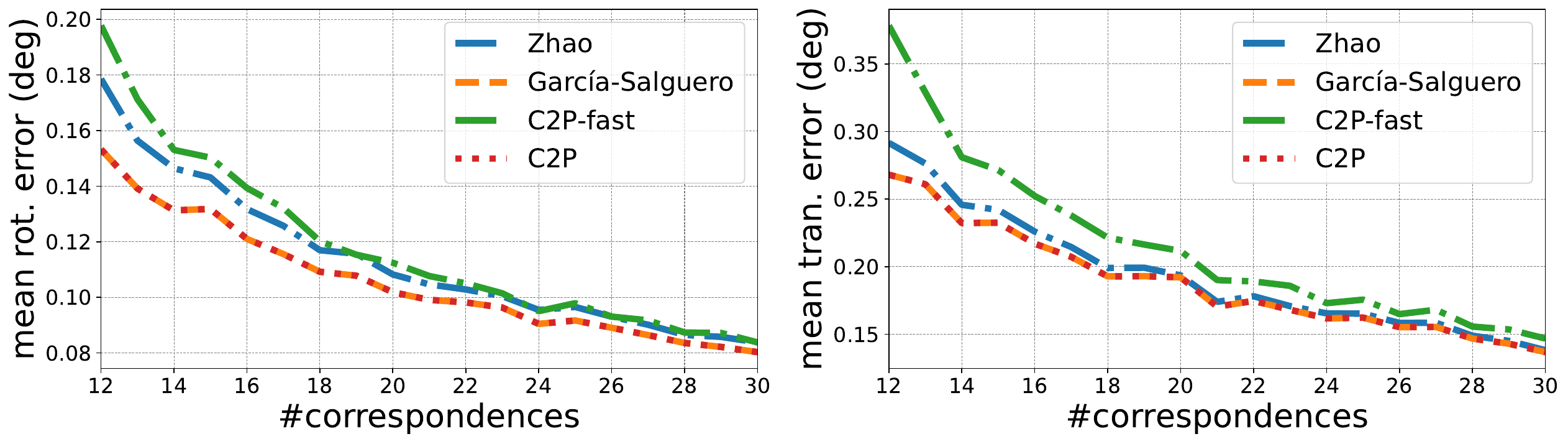} &
        \includegraphics[width=8cm]{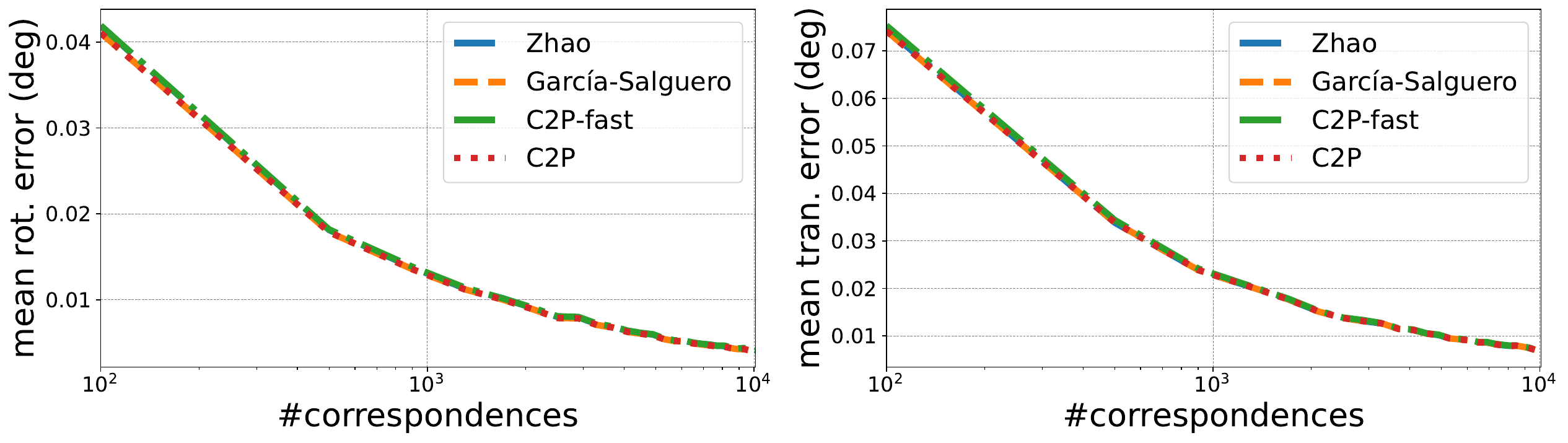}\\
        (a) Regime R1: $n\in[12, 30]$ 
        & 
        (b) Regime R2: $n\in[10^2, 10^4]$
    \end{tblr}
    \caption{\textbf{Accuracy vs number of correspondences.} (a) For small $n$, both \name and the method of \citet{garciasalguero2022tighter} perform the best, while \name is faster than \cite{garciasalguero2022tighter}. We argue that redundant constraints, such as those leveraged in these methods, are particularly helpful in this regime. (b) As $n$ increases (right) both our \name and \nameFast scale better than \cite{zhao2022nonmin,garciasalguero2022tighter} while having the same accuracy.}
    \label{fig:acc_vs_n}
\end{figure*}

%% file: figures/pure_rot/pure_rot.tex
\begin{figure}
    \centering
    \includegraphics[width=0.6\linewidth]{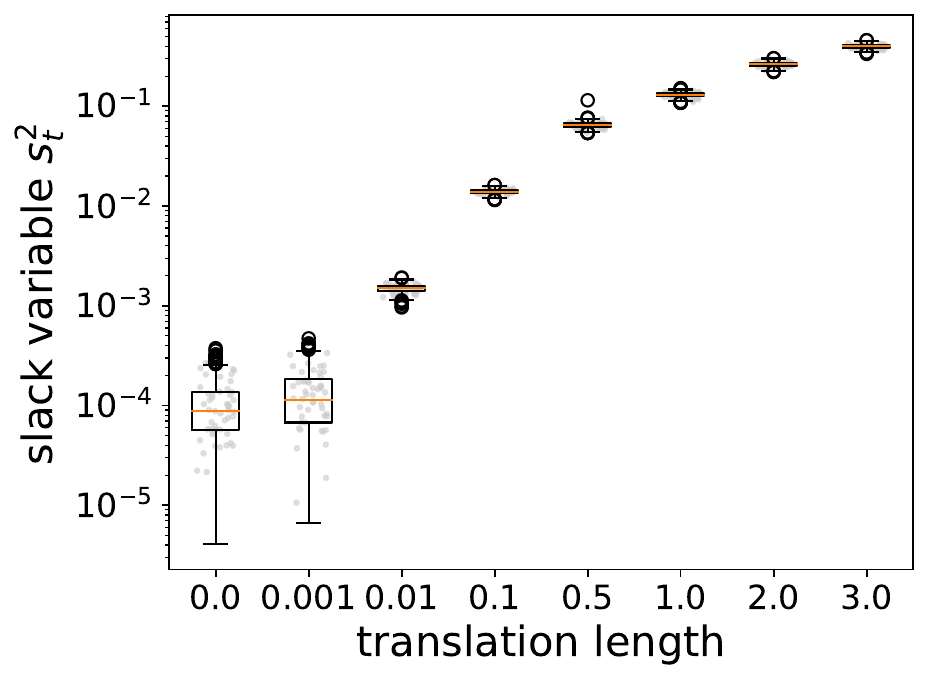}
    \caption{\textbf{$s_t^2$ as pure rot. metric.} As the translation magnitude decreases, the accuracy of the estimated $\bt$, declines due to its diminished impact on the cost function. Our slack variable, $s_t^2$, allows us to directly identify when the magnitude of $\bt$ is $\leq10^{-3}$ of relative to the scene's scale through  
    simple thresholding on $s_t^2$.%
    }
    \label{fig:pure_rot}
\end{figure}

%% file: sec/5_conclusions.tex
\section{Conclusion and Limitations}\label{sec:conclusion}
In this paper, we introduced \name, a novel method for estimating the relative pose that, for the first time in the literature, does not need posterior disambiguation.
Our approach efficiently constrains the parameter space during optimization using the necessary and sufficient geometric and manifold constraints, resulting in better runtime with more correspondences and a better balance between accuracy and efficiency compared to the alternatives. Additionally, our method is certifiably globally optimal and can directly detect near-pure rotational motions. \name, however, as of now, cannot deal with outliers, which is left for future work. 

{\footnotesize \PAR{Acknowledgements.} This work was supported by the Ministerio de Universidades Scholarship FPU21/04468, the Spanish Government (PID2021-127685NB-I00 and TED2021-131150B-I00) and the Arag\'on Government (DGA T45\_23R). \par}

%% file: sec/X_suppl.tex
\clearpage
\setcounter{page}{1}
\maketitlesupplementary

\appendix

\iftoggle{cvprfinal}{In this supplementary material, we provide additional details of our method in \Cref{sec:supp_details}, proofs in \Cref{sec:supp_tightness,sec:supp_deriv_midpoints} and experiments in \Cref{sec:supp_exp}.}{In this supplementary material, we provide additional details, proofs and experiments. Please note that we have included the ``References'' section again at the end, as new references are used in the supplementary.}

\section{Additional details}\label{sec:supp_details}

\subsection{Averaging of data-dependent constraints}\label{app:supp_avg}
We provide here compact expressions for averaging the data-dependent coefficients of the quadratic terms stemming from \cref{eq:qcqp_rotdis,eq:qcqp_tdis}. We will use the notation $\ba^{(j)}$ to refer to the $j-th$ element of a vector $\ba$.

\PAR{Rotation.}
For clarity, we recall \cref{eq:qcqp_rotdis}:
\begin{align}
    \aug{\bff}_1\bE^\top\skeww{\bt}\aug{\bff}_0 - s_r^2 = 0~,\label{eq:qcqp_rotdis_supp}
\end{align}
The quadratic terms for one correspondence $(\bff_0, \bff_1)$ are:
\begin{equation}
\resizebox{0.84\linewidth}{!}{%
    ${\displaystyle
    \begin{aligned}
          &\bff_0^{(0)}\bff_1^{(0)}\be^{(3)}\bt^{(2)} + \bff_0^{(0)}\bff_1^{(1)}\be^{(4)}\bt^{(2)} + \bff_0^{(0)}\bff_1^{(2)}\be^{(5)}\bt^{(2)} \\
        +~&\bff_0^{(1)}\bff_1^{(0)}\be^{(6)}\bt^{(0)} + \bff_0^{(1)}\bff_1^{(1)}\be^{(7)}\bt^{(0)} + \bff_0^{(1)}\bff_1^{(2)}\be^{(8)}\bt^{(0)} \\
        +~&\bff_0^{(2)}\bff_1^{(0)}\be^{(0)}\bt^{(1)} + \bff_0^{(2)}\bff_1^{(1)}\be^{(1)}\bt^{(1)} + \bff_0^{(2)}\bff_1^{(2)}\be^{(2)}\bt^{(1)} \\
        -~&\bff_0^{(0)}\bff_1^{(0)}\be^{(6)}\bt^{(1)} - \bff_0^{(0)}\bff_1^{(1)}\be^{(7)}\bt^{(1)} - \bff_0^{(0)}\bff_1^{(2)}\be^{(8)}\bt^{(1)} \\
        -~&\bff_0^{(1)}\bff_1^{(0)}\be^{(0)}\bt^{(2)} - \bff_0^{(1)}\bff_1^{(1)}\be^{(1)}\bt^{(2)} - \bff_0^{(1)}\bff_1^{(2)}\be^{(2)}\bt^{(2)} \\
        -~&\bff_0^{(2)}\bff_1^{(0)}\be^{(3)}\bt^{(0)} - \bff_0^{(2)}\bff_1^{(1)}\be^{(4)}\bt^{(0)} - \bff_0^{(2)}\bff_1^{(2)}\be^{(5)}\bt^{(0)} \\
        -~&s_r^2 = 0~,%
    \end{aligned}
    }$}%
    \label{eq:rot_coeff_terms}
\end{equation}
where, $\be\ceqq\vc(\bE^\top)$, as defined in the main paper.

With this ordering, the coefficients of the first nine terms of \cref{eq:rot_coeff_terms} (from $\bff_0^{(0)}\bff_1^{(0)}$ to $\bff_0^{(2)}\bff_1^{(2)}$) can be computed as $\vc(\bff_1\bff_0^\top)$,
which is a nine-dimensional vector (one element per coefficient). Thus, the averaging of the terms across $n$ correspondences $\{\corr_{0,i}, \corr_{1,i}\}_{i=1}^{n}$ can be expressed as:
\begin{equation}
    \vc\left(\frac{1}{n}\sum_{i=1}^n\bff_{1,i}\bff_{0,i}^\top\right) \label{eq:avg_rot}~.
\end{equation}
Furthermore, the subsequent nine quadratic terms have the same (but negated) coefficients. Thus, the values of \cref{eq:avg_rot} can be reused for these coefficients.

\PAR{Translation.} For clarity, we recall \cref{eq:qcqp_tdis}:
\begin{equation}
    h\aug{\bff}_0^\top\bt - h\aug{\bff}_1^\top\bq - s_t^2 = 0~,
\end{equation}
The quadratic terms for one correspondence $(\bff_0, \bff_1)$ are:
\begin{equation}
\begin{alignedat}{3}
      & h\bff_0^{(0)}\bt^{(0)} ~&&+~ h\bff_0^{(1)}\bt^{(1)} ~&&+~ h\bff_0^{(2)}\bt^{(2)} \\
    - & h\bff_1^{(0)}\bq^{(0)} ~&&-~ h\bff_1^{(1)}\bq^{(1)} ~&&-~ h\bff_1^{(2)}\bq^{(2)} \\
    - & s_t^2 = 0
\end{alignedat}
\end{equation}
In this case the coefficients of the quadratic terms are directly given by the bearings. Thus, the average for the first three coefficients can be computed as $1/n \sum_i^n\bff_0$, and as $-1/n \sum_i^n\bff_1$ for the subsequent three coefficients.

\subsection{Appropriate scaling of the solution estimates}\label{sec:supp_scaling}
As explained in \cref{sec:recovery}, we extract the solution estimates from the dominant singular vector, denoted as $\bv_0$, of $\bXEtqh\ceqq\bXopt_{[1:16,1:16]}$. Following the ordering and notation of the main paper, this corresponds to 
$\bv_0=[(\opt{\be})^\top, (\opt{\bt})^\top, (\opt{\bq})^\top, h]^\top$.
However, the norm constraints enforced during the optimization, namely $\bt^\top\bt=1$, $\bq^\top\bq=1$ and $\trace(\bE\bE^\top)=2$, apply to $\bXopt$ and not to $\bv_0$. Consequently, we cannot assume that the elements of this dominant vector will be scaled appropriately even after multiplying it with its singular value. The solution to this is straightforward: we separately normalize the vectors $\bt$ and $\bq$ to make them unit vectors, and scale $\be\ceqq\vc(\bE^\top)$ such that its nonzero singular values equal $1$ (in practice, we use the SVD of $\bE$ for greater precision). 
Finally, we leverage the absence of products between the slack variable $s_t$ and the rest of the parameters in \crefprob{word:qcqp} to directly read $s_t^2$ from its corresponding diagonal entry in $\bXopt$, thus avoiding the need to factorize $\bXopt$ to obtain its value.

\input{figures/supp/homog_vs_t/homog_vs_t}

\subsection{Pure rotations and numerical accuracy}\label{sec:supp_num_acc}

Under pure rotations, considering an optimal essential matrix, $\opt{\bE}$, any pair of translation vectors $\bt, \bq\in\Stwo$ satisfying the definition $\bq\ceqq\bR^\top\bt$ will minimize the sum of squared epipolar errors. Here, $\bR\in\SOthree$ represents one of the two rotation matrices corresponding to $\opt{\bE}$. Given that both $\bt, \bq$ belong to $\Stwo$, one might expect to find two additional singular vectors corresponding to nonzero singular values, in addition to the three singular vectors metioned in \cref{sec:recovery}. However, we empirically verified that four additional singular vectors appear instead. We observed the same phenomenon in \cite{zhao2022nonmin,garciasalguero2022tighter}. This phenomenon likely occurs because the constraints apply only to the optimal matrix $\bXopt$. Therefore, the elements in the singular vectors of $\bXopt$ do not need to satisfy the norm constraints (\eg that $\bt,\bq$ belong to $\Stwo$) to still minimize the cost function. This may explain the similar behavior noted in \cite{Briales_2018_CVPR} regarding pure rotations. 

Importantly, in our case,  the correct solution can be extracted from the dominant singular vector thanks to \cref{eq:qcqp_rotdis}, which enforces a larger singular value corresponding to the vector containing the solution. However, for pure rotations and in absence of noise, the component in the (unit) singular vector corresponding to the homogenization variable, $h$, dominates the rest, being close to $\sim1$. This predominance reduces the numerical accuracy of the other estimates ($\be,\bt$ and $\bq$). Since this behavior is only present in a noise-free scenario, we can use a strict threshold in the slack variable $s_t^2$ (we use $10^{-4}$) to detect such scenario. Consequently, only in this case, we extract the solution from the dominant singular vector of the submatrix corresponding only to $\be,\bt$ and $\bq$, leveraging the previous numerically-inaccurate solution to just correct the sign of this new numerically-accurate solution, if necessary. This behavior is shown in \cref{fig:homog_vs_t}.

\section{Tightness of \cite{zhao2022nonmin} when the SDP solution is rank-2}\label{sec:supp_tightness}

In \cite[Eq. 11]{zhao2022nonmin} the following QCQP is considered:
\PAR{Problem \setword{QCQP-Z}{prob:qcqp-z}}
\begin{align}
    \min_{\bE,\bt} \quad& \be^\top\bC\be~, \\
    \textrm{s.t.} \quad & 
    \bE\bE^\top = \skeww{\bt}\skeww{\bt}^\top,
    \quad\bt^\top\bt=1~.
\end{align}
The tightness conditions in \cite[Th. 2]{zhao2022nonmin} assume that tightness of the semidefinite relaxation imply $\rank(\bXopt)=1$, where $\bXopt$ represents the optimal solution of the SDP. In this section, we adapt \cref{th:qcqp} to extend \mbox{\cite[Th. 2]{zhao2022nonmin}} and show that \crefprob{prob:qcqp-z} can also be tight when $\rank(\bXopt)=2$.

\input{proofs-theorems-algos/th_proof_tight_cond_sdr_QCQP-Z}

\section{Algebraic derivation of  \Cref{eq:midpoints}}\label{sec:supp_deriv_midpoints}

Given estimates of the relative rotation and translation $(\bR, \bt)$, and a correspondence $(\bff_{0}, \bff_{1})$, the midpoint method triangulates the corresponding 3D point $\bp\in\bb{R}^3$. It identifies this point as the midpoint (mean) of the common perpendicular to the two rays originating from the bearings \cite{beardsley1994navigation}. Specifically, it determines the norms $\lambda_0,\lambda_1\in\bb{R}$ of the 3D points, $\bp_0\ceqq\lambda_0\bff_0, \bp_1\ceqq\lambda_1\bff_1$, in each camera reference system, that minimize the squared error $\lVert\bp_0 - (\bR\bp_1+\bt)\rVert^2$:
\begin{equation}
    \lambda_0, \lambda_1 =
    \arg\min_{\lambda_0, \lambda_1} 
    \lVert 
    \lambda_0\bff_0 - (\lambda_1\bR\bff_1 + \bt)
    \rVert^2~.\label{eq:min_midpoints}
\end{equation}
If the 3D points and their midpoint (mean) satisfy the cheirality constraints, 
both norms $\lambda_0$ and $\lambda_1$ will be positive. Otherwise, at least one of the norms will be estimated as negative \cite{tron2017essential}. As will be shown, it is not necessary to explicitly compute $\lambda_0$ and $\lambda_1$ to estimate their signs.

The rays of ideal, noise-free correspondences meet in a 3D point, satisfying $\lambda_0\bff_0 - \lambda_1\bR\bff_1 = \bt$, or in matrix form:
\begin{align}
    \underbrace{\begin{bmatrix}
        \bff_0 & -\bR\bff_1
    \end{bmatrix}}_{\bA\in\bb{R}^{3\times2}}
    \begin{bmatrix}
        \lambda_0 \\
        \lambda_1
    \end{bmatrix}
    = \bt
\end{align}
In practice, we minimize the squared errors. As such, an equivalent solution to \cref{eq:min_midpoints} is given as the solution to the system $\bA^\top\bA~ [\lambda_0,~\lambda_1]^\top=\bA^\top\bt$:
\begin{equation}
    \begin{bmatrix}
        1 & -\bff_0^\top\bR\bff_1\\
        -\bff_0^\top\bR\bff_1 & 1\\
    \end{bmatrix}
    \begin{bmatrix}
        \lambda_0 \\
        \lambda_1
    \end{bmatrix} = 
    \begin{bmatrix}
        \bff_0^\top \\
        -(\bR\bff_1)^\top
    \end{bmatrix}
    \bt~,
    \label{eq:midpoints_lin}
\end{equation}
where we have used that $(\bff_k)^\top\bff_k=1,~k\in\{0, 1\}$ since $\bff_0,~\bff_1\in\Stwo$. Expanding \Cref{eq:midpoints_lin} leads to:
\begin{align}
    \lambda_0 - \lambda_1\bff_0^\top\bR\bff_1 &= \bff_0^\top\bt~,\\
    \lambda_1 - \lambda_0\bff_0^\top\bR\bff_1 &= -(\bR\bff_1)^\top\bt~.
\end{align}
which leads to the equivalent equations:
\begin{alignat}{3}
    s^2\lambda_1 =&& 
    -(\bR\bff_1)^\top\bt &+ (\bff_0^\top\bR\bff_1)(\bff_0^\top\bt)~,\label{eq:mid_aux1}\\
    s^2\lambda_0 =&& 
    \bff_0^\top\bt &- (\bff_0^\top\bR\bff_1)((\bR\bff_1)^\top\bt)~.\label{eq:mid_aux2}
\end{alignat}
where 
\begin{equation}
    s^2 \ceqq 1 - (\bff_0^\top\bR\bff_1)^2
    =\sin^2\angle(\bff_0, \bR\bff_1)~.
\end{equation}
Since $s^2\geq0$, this implies that the RHS of \cref{eq:mid_aux1,eq:mid_aux2} must be positive for $\lambda_0,\lambda_1$ to be positive too:
\begin{alignat}{3}
    &&-(\bR\bff_1)^\top\bt &+ (\bff_0^\top\bR\bff_1)(\bff_0^\top\bt) &&> 0~,\label{eq:mid_aux3}\\
    &&\bff_0^\top\bt &- (\bff_0^\top\bR\bff_1)((\bR\bff_1)^\top\bt) &&> 0 ~,\label{eq:mid_aux4}
\end{alignat}
Lastly, to express \cref{eq:mid_aux3,eq:mid_aux4} in compact form, we can use the property of the cross product:
\begin{equation}
    (\ba\times\bbb)\cdot(\bc\times\bd)
    =
    (\ba\cdot\bc)(\bbb\cdot\bd)-(\ba\cdot\bd)(\bbb\cdot\bc)~,
\end{equation}
for any $\ba,\bbb,\bc,\bd\in\bb{R}^3$, and with $\ba\cdot\bbb=\ba^\top\bbb$ representing the dot product between any vectors $\ba,\bbb$. With this property, we reach the inequalities:
\begin{alignat}{3}
    (\bR\bff_1\times\bff_0)&\cdot(\bff_0\times\bt)&&>0~,\\
    (\bR\bff_1\times\bff_0)&\cdot(\bR\bff_1\times\bt)&&>0~,
\end{alignat}
corresponding thus to \cref{eq:midpoints}.

\section{Additional experiments}\label{sec:supp_exp}

\subsection{Accuracy vs noise and translation magnitude}
In \Cref{fig:supp_synthetic}, we show additional synthetic experiments. In \cref{fig:supp_synthetic}(a)-(e), we verify that the conclusions drawn in \cref{fig:acc_vs_n} are consistent across different of levels of noise. In regimes with a small number of points (\cref{fig:supp_synthetic}(a),(c)), our \name and \cite{garciasalguero2022tighter} perform the best. Notably, \name is faster than \cite{garciasalguero2022tighter} (see \cref{fig:runtime}) and slightly outperforms it in estimating translation (\cref{fig:supp_synthetic}(c)). In regimes with a large number of points, the accuracy of our faster version of \name is on-par with \name itself and \cite{garciasalguero2022tighter}. In \cref{fig:supp_synthetic}(e), we fix the number of points at $1000$ while varying the noise level and observe the same behavior. Finally, in \cref{fig:supp_synthetic}(f), we demonstrate that \name performs as well as \cite{zhao2022nonmin, garciasalguero2022tighter} when varying the scale of the translation w.r.t. the scene, and unlike them, \name is also capable of directly detecting near-pure rotational motions and does not need posterior disambiguation step.

\input{figures/supp/synthetic_accuracy}

\subsection{Real-data}
Following \cite{zhao2022nonmin,garciasalguero2022tighter}, we test our method on all the sequences from \citet{strecha2008benchmark}. We generate 97 wide-baseline image pairs by grouping adjacent images. For each image pair, correspondences are extracted using DoG + SIFT \cite{lowe2004distinctive}. We then use the RANSAC implementation of OpenGV \cite{kneip2014opengv} to filter out wrong correspondences, setting the inlier threshold to $5$ pix, which we found sufficient given the images resolution of $3072\times2048$ pix. The performance of \name{\small\textsf{(-fast)}} and \cite{zhao2022nonmin,garciasalguero2022tighter} is shown in \cref{fig:strecha}. The results align with those from the synthetic experiments. \nameFast is the fastest among all the methods. However, \nameFast is not always tight, resulting in a slight loss of accuracy with respect to the alternatives. On the other hand, our \name is significantly more accurate than \cite{zhao2022nonmin} and is on-par with \cite{garciasalguero2022tighter}. Additionally, our \name is 40\% faster than \cite{garciasalguero2022tighter} on average.

\input{figures/supp/real_data/strecha}

%% file: figures/supp/homog_vs_t/homog_vs_t.tex
\begin{figure*}
    \centering
    \begin{tblr}{
            colspec={ccc},
            rowsep=1pt,
            colsep=8pt,
        }
        \includegraphics[height=3cm]{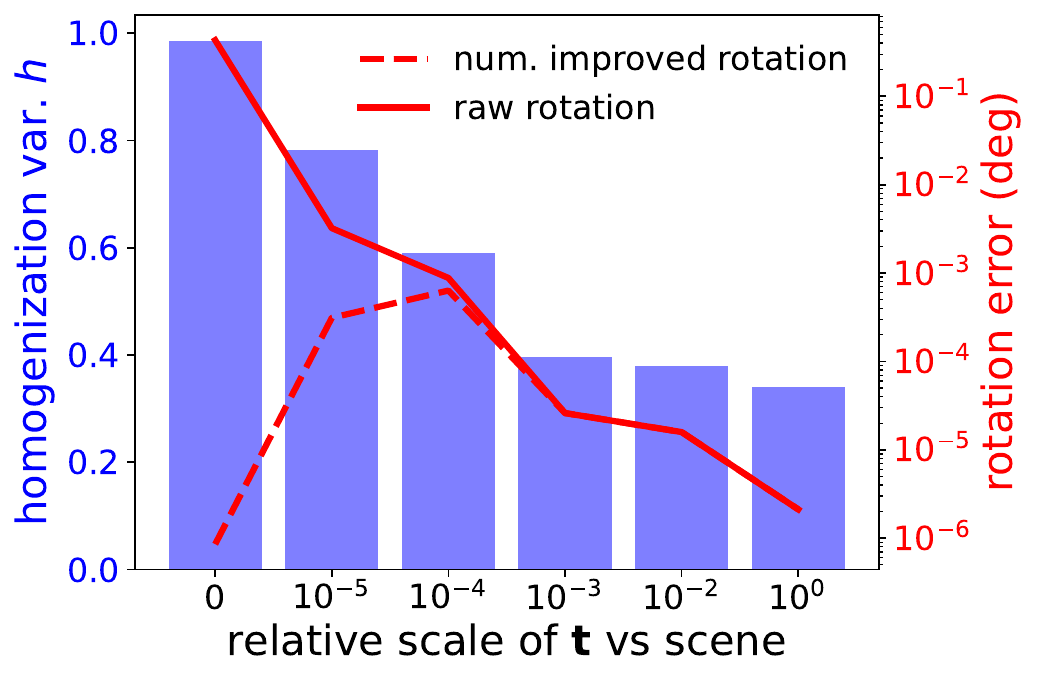} &
        \includegraphics[height=3cm]{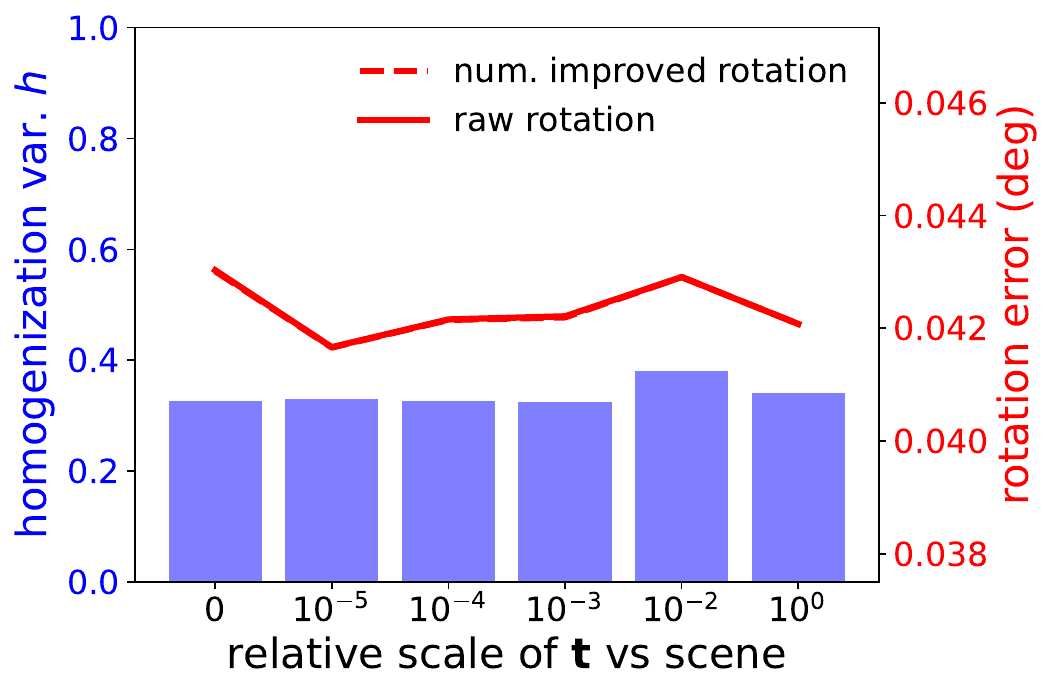} & 
        \includegraphics[height=3cm]{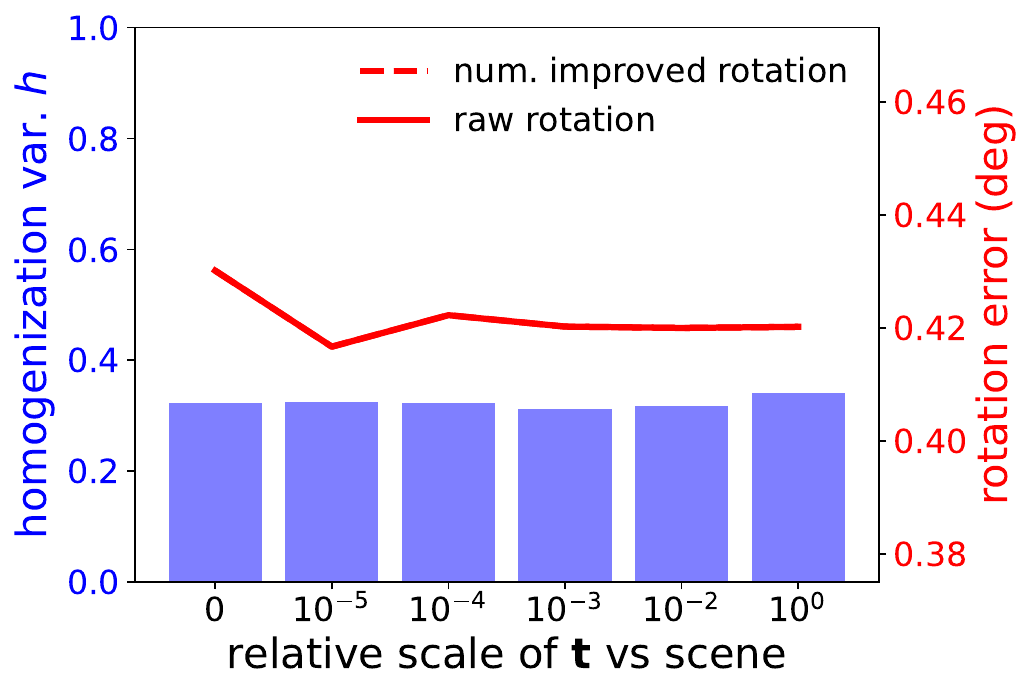} \\
        \footnotesize{(a) noise = 0 pix} & \footnotesize{(b) noise = 1 pix} & \footnotesize{(c) noise = 10 pix}
    \end{tblr}
    \caption{\textbf{Noise-free pure rotation scenarios.} For our method, an edge case consists of noise-free pure rotational motions. In noise-free scenarios (a) the estimate of $h$ within the dominant singular vector---which contains the rest of the solution estimates---approaches $1$ in near-pure rotational motions (when the relative scale is $<10^{-4}$), negatively affecting the numerical accuracy of the other estimates. As can be seen, we effectively address this by using the dominant singular vector from the submatrix excluding $h$. However, despite the effectiveness of this solution, this numerical issue is not present in practical scenarios (b), (c), where noise affects the observations. These visualizations depict results averaged across 1000 different random instances of the same synthetic scenarios considered in the main paper.}
    \label{fig:homog_vs_t}
\end{figure*}

%% file: proofs-theorems-algos/th_proof_tight_cond_sdr_QCQP-Z.tex
With a similar notation as in the main paper, let us define the following auxiliary variables:
\begin{equation}
    \bXopte\ceqq\bXopt_{[1:9,1:9]}~,  \quad 
    \bXoptt\ceqq\bXopt_{[10:12,10:12]}~.
\end{equation}

\begin{theorem}
    The semidefinite relaxation of \crefprob{prob:qcqp-z} is tight if and only if $\rank(\bXopt)\in[1, 2]$, and its submatrices $\bXopte$ and $\bXoptt$ are rank-1.
\end{theorem}

\begin{proof}
    \newcommand{\vect}{\begin{bmatrix}
            \be \\
            \bt
        \end{bmatrix}}
    \newcommand{\vecf}{\begin{bmatrix}
            \be \\
            -\bt
        \end{bmatrix}}
    For the \emph{only if} direction, assume the relaxation is tight. Then, following \cite{Briales_2018_CVPR}, we can find $\bXopt$ in the convex hull of the linearly independent rank-1 solutions to the relative pose problem\footnote{Note that the outer products of the negative counterparts, $[-\be^\top,-\bt^\top]$ and $[-\be^\top,\bt^\top]$, are not included, as they yield the same outer product.}:
    \begin{align}
        &\opt{\bX} \ceqq 
        \alpha_0\bx_0\bx_0^\top + \alpha_1\bx_1\bx_1^\top, \label{eq:outer_ref}\\
        &\bx_0 \ceqq \vect, \quad
        \bx_1 \ceqq \vecf,
    \end{align}
    where $\alpha_0,\alpha_1$ are non-negative scalars such that $\alpha_0+\alpha_1=1$. This last condition ensures that the cost is optimal, \ie, $\trace(\bC_0\bXopt)=\be^\top\bC\be$, and that the resulting matrix $\bXopt$ is feasible. To see this, we can expand \cref{eq:outer_ref}:
    \begin{equation}
        \opt{\bX} = 
        \begin{bmatrix}
            (\alpha_0+\alpha_1)\be\be^\top 
            & (\alpha_0-\alpha_1)\be\bt^\top\\
            (\alpha_0-\alpha_1)\bt\be^\top & 
            (\alpha_0+\alpha_1)\bt\bt^\top
        \end{bmatrix}~,\label{eq:XZhao_expand}
    \end{equation}
    to verify that $\alpha_0+\alpha_1=1$ is needed to satisfy the norm constraint $\bt^\top\bt=1$ (the rest of the constraints are satisfied for any valid combination of $\alpha_0$ and $\alpha_1$). This reveals that when the semidefinite relaxation is tight, the diagonal (upper-left and bottom-right) block matrices are rank-1 and that $\rank(\bXopt)\in\{1,2\}$. Specifically\footnote{In practice, off-the-shelf SDP solvers \cite{sturm1999sedumi, toh1999sdpt3, yamashita2010sdpa7} return a rank-2 block-diagonal solution \cite{garciasalguero2022tighter}, which corresponds to setting $\alpha_0=\alpha_1=0.5$ in \cref{eq:outer_ref,eq:XZhao_expand}.}, $\rank(\bXopt)=1$ when $\alpha_0=0$ and $\alpha_1=1$ or when $\alpha_0=1$ and $\alpha_1=0$. Otherwise $\rank(\bXopt)=2$.

    For the \emph{if} part, we build upon \cite[Theorem 2]{zhao2022nonmin}. Since $\bXopt$ is a positive semidefinite (PSD) matrix, $\bXopte$ and $\bXoptt$ are also PSD as they are principal submatrices of $\bXopt$ \cite{strang1980its}. Given that $\bXopte$ and $\bXoptt$ are both rank-1 matrices, it follows that there exist two vectors $\opt{\be}\in\bb{R}^9$ and $\opt{\bt}\in\bb{R}^3$ that fulfill the primal problem's constraints and satisfy $\opt{\be}(\opt{\be})^\top=\bXopte$ and $\opt{\bt}(\opt{\bt})^\top=\bXoptt$.

    Regarding the rank of $\bXopt$, since it is PSD, it can be factorized as $\bXopt=\bL\bL^\top$, where $\bL\in\bb{R}^{12\times r}$ and $r\ceqq\rank(\bXopt)$. Thus, to satisfy the rank-1 property of $\bXopte$ and $\bXoptt$, each column $k$ of $\bL$ must be given by: $[a_k\be^\top,b_k\bt^\top]^\top$, for some scalars $a_k,b_k\in\bb{R}$. This constraint limits the rank of $\bXopt$ to at most $2$, as any additional column in $\bL$ would be a linear combination of the existing ones. Therefore, since $\bXopt$ must be feasible, this implies that $\rank(\bXopt)\in[1,2]$ and that it is a convex combination of the two linearly independent solutions, stemming from \cref{eq:outer_ref}, and thus the relaxation is tight.
\end{proof}

%% file: figures/supp/synthetic_accuracy.tex
\begin{figure*}
    \centering
    \begin{tblr}{
        colspec={cc},
        rowsep=2pt,
        colsep=2pt,
        }
        \includegraphics[height=2.4cm]{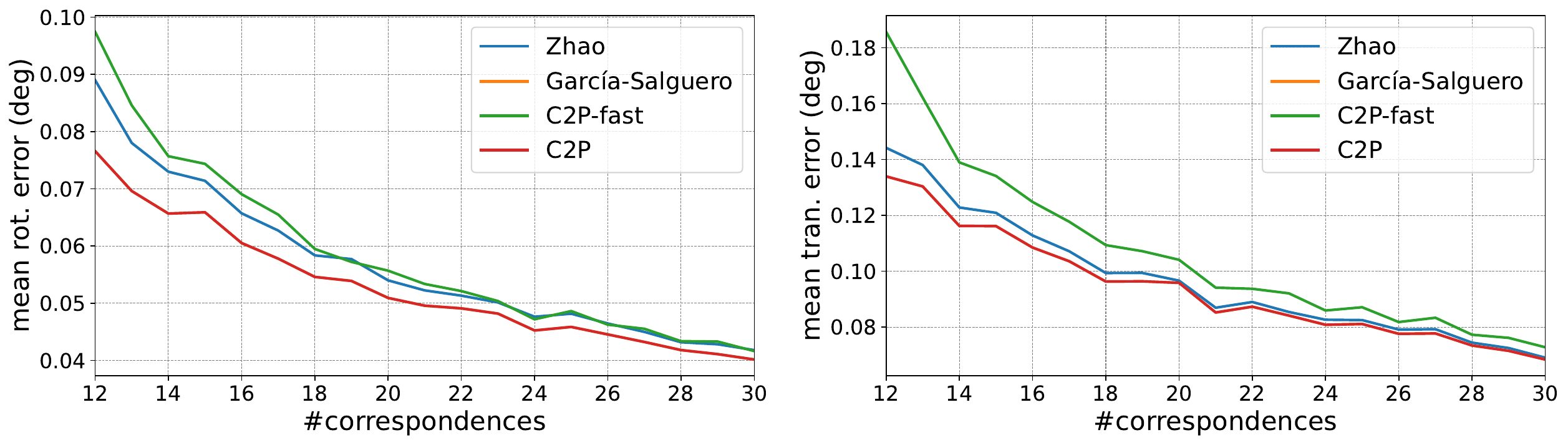} &
        \includegraphics[height=2.4cm]{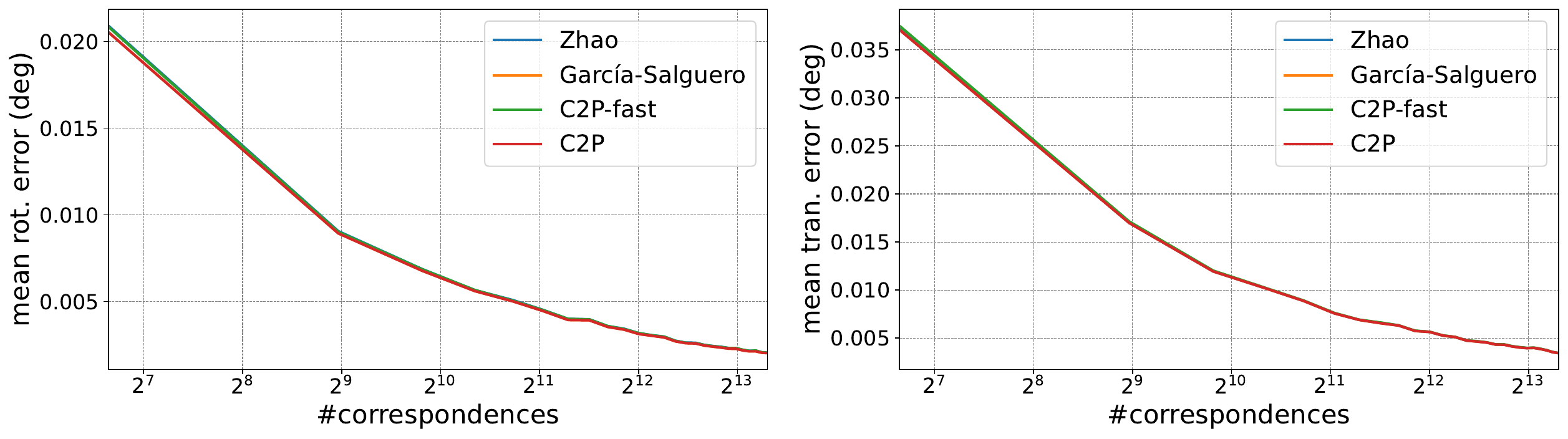} \\
        \footnotesize{(a) noise = $0.5$ pix, $n\in[12,30]$} & \footnotesize{(b) noise = $0.5$ pix, $n\in[10^2, 10^4]$} \\
        \includegraphics[height=2.4cm]{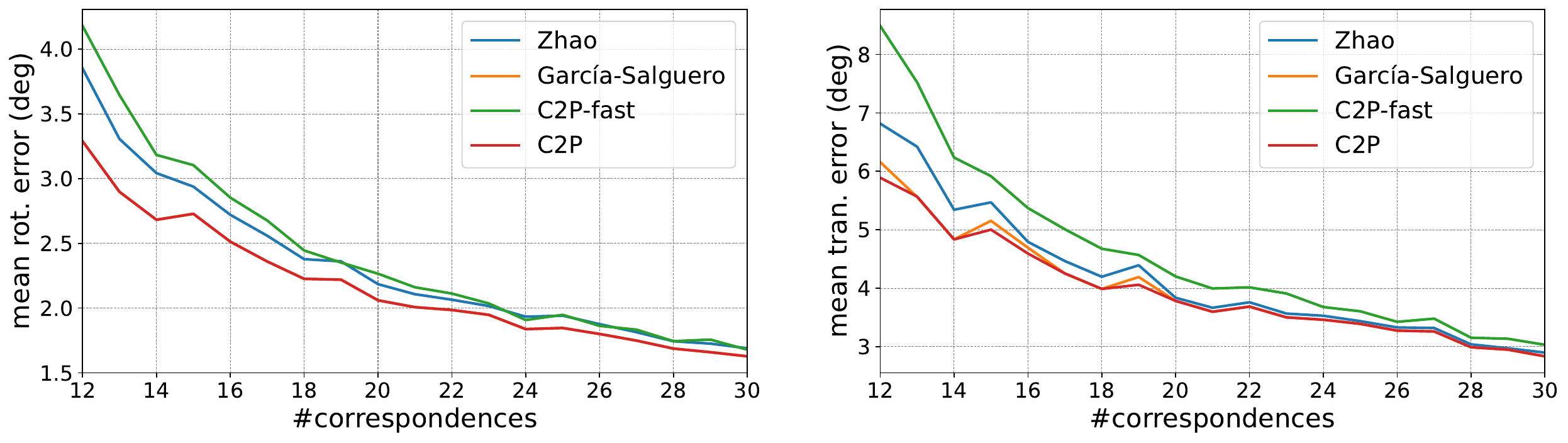} &
        \includegraphics[height=2.4cm]{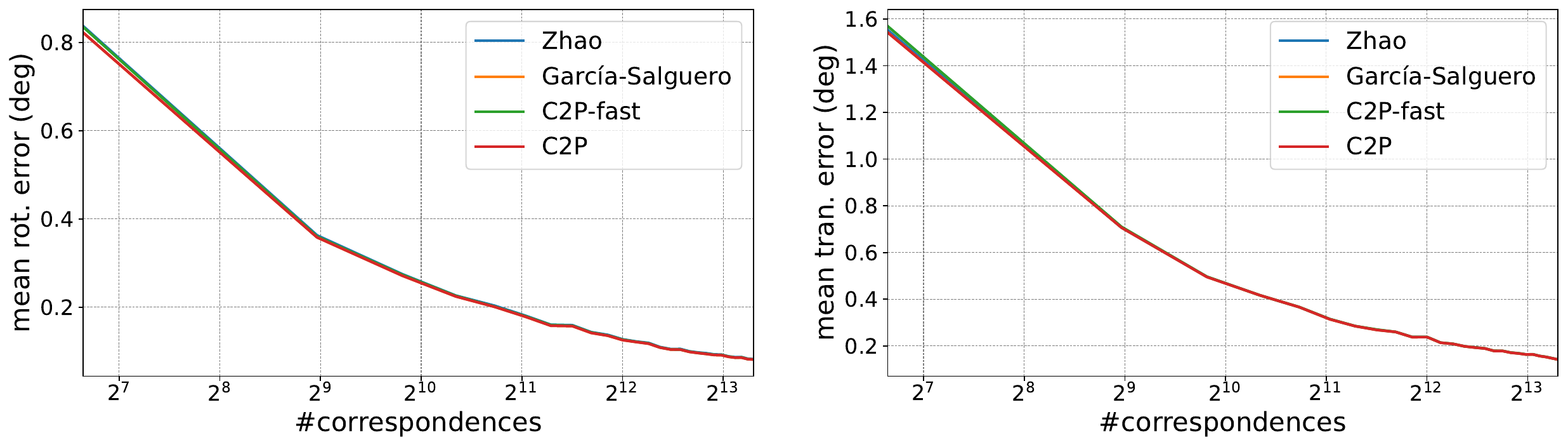} \\
        \footnotesize{(c) noise = $20$ pix, $n\in[12,30]$} & \footnotesize{(d) noise = $20$ pix, $n\in[10^2, 10^4]$} \\
        \includegraphics[height=4.4cm]{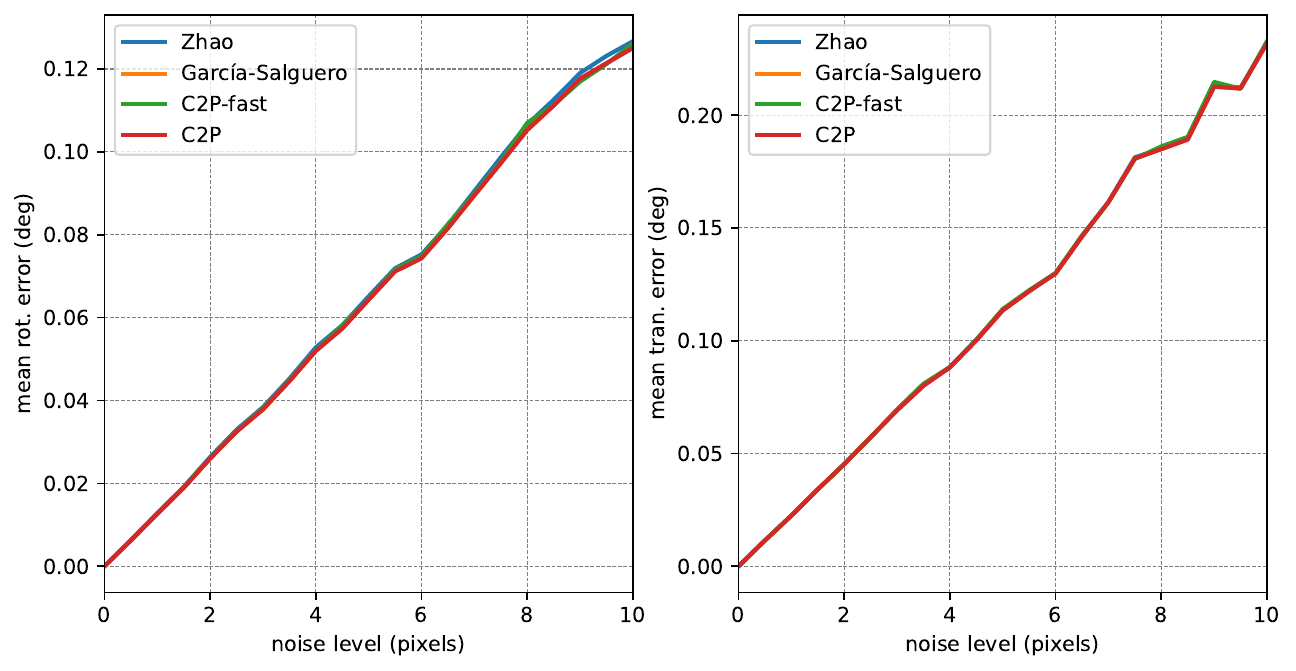} & 
        \includegraphics[height=4.4cm]{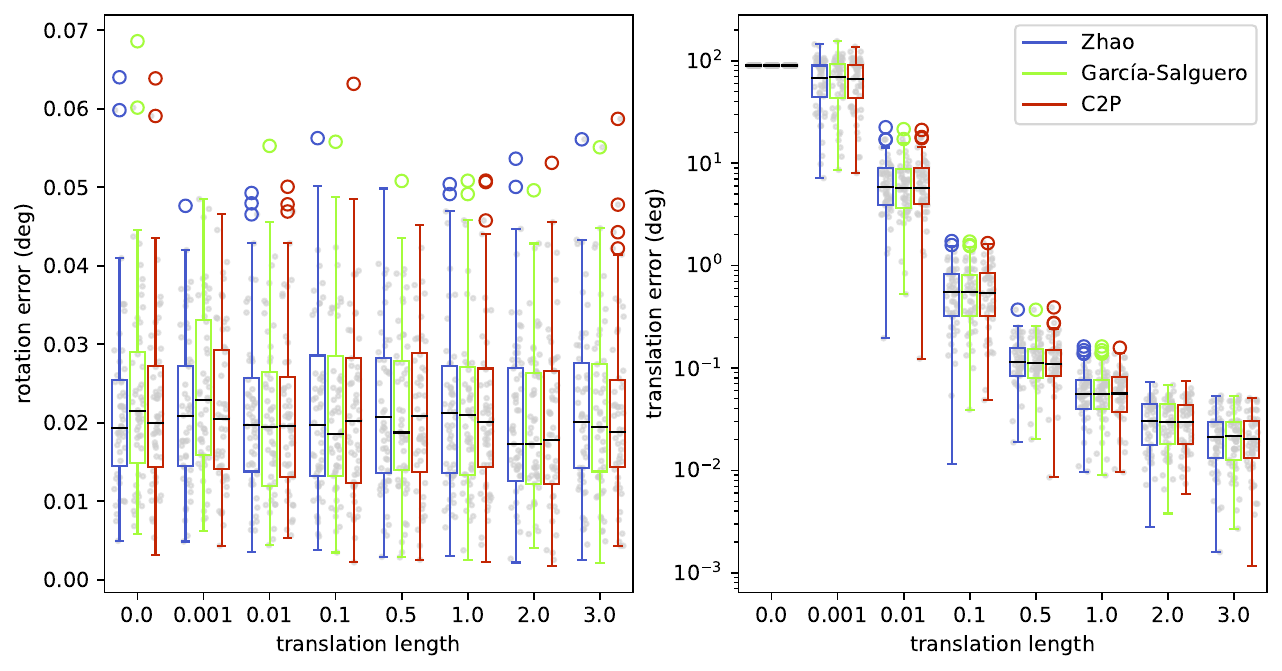} \\
        \footnotesize{(e) accuracy vs noise level and $n=1000$} & \footnotesize{(f) accuracy vs translation length}
    \end{tblr}
    \caption{\textbf{Additional synthetic experiments.} We evaluate our proposed \name and \nameFast under various conditions: (a)-(d) number of correspondences, (a)-(e) noise levels, and (f) relative translation scale w.r.t. scene. As shown in (b) and (d), \nameFast is well-suited for scenarios where $n>10^3$, performing on-par with \name and \cite{garciasalguero2022tighter, zhao2022nonmin}, while being faster (\cref{fig:runtime}). With fewer correspondences, as shown in (a) and (c), \name outperforms \cite{zhao2022nonmin}, slightly surpassing the accuracy of \cite{garciasalguero2022tighter} in estimating the translation, while also being faster. The same conclusions are reached when varying the noise levels (e). Finally, in (f) we show that \name performs as well as \cite{zhao2022nonmin, garciasalguero2022tighter} when varying the scale of the translation relative to the scene, and unlike them, \name is capable of directly detecting near-pure rotational motions.}
    \label{fig:supp_synthetic}
\end{figure*}

%% file: figures/supp/real_data/strecha.tex
\begin{figure*}
    \centering
    \includegraphics[width=10.5cm,valign=m]{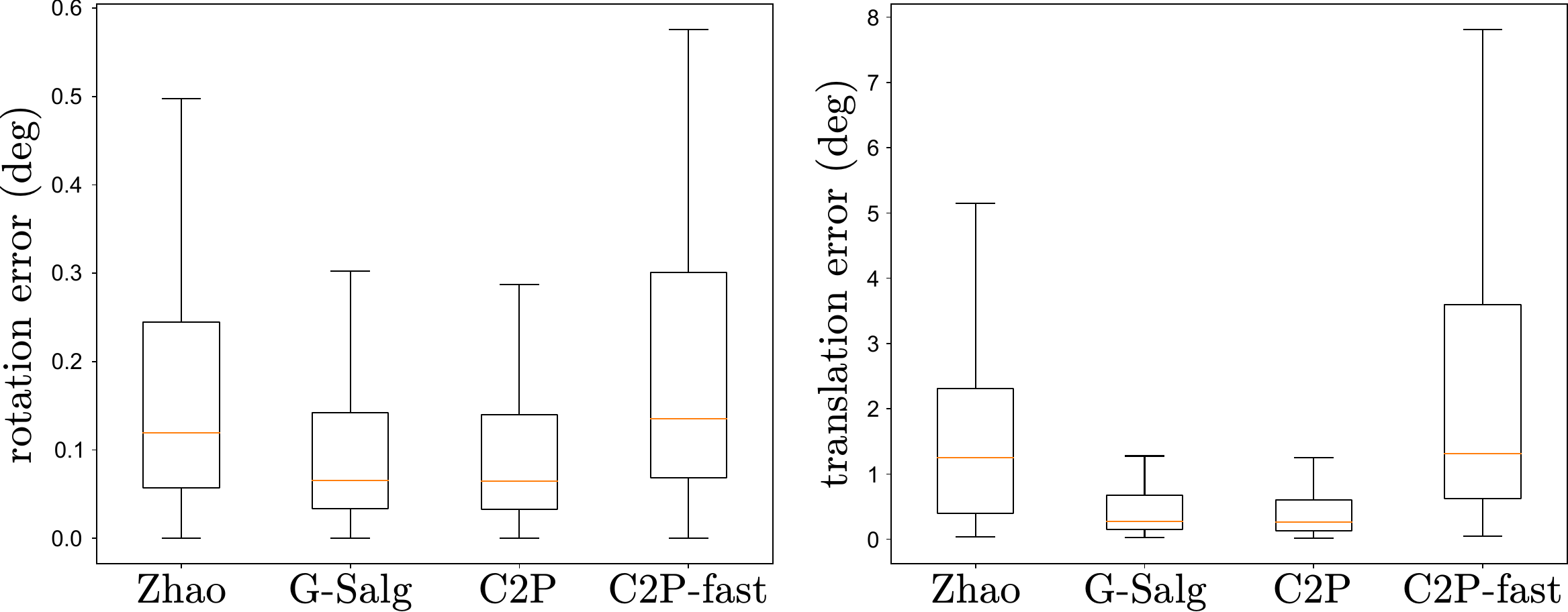}
    \hfill
    \begin{tblr}{
        colspec={lr},
        rowsep=1pt,
    }   
        \toprule
        Method      & Time (ms) \\
        \midrule
        \citet{zhao2022nonmin}            &  9.2 \\
        \citet{garciasalguero2022tighter} &  14.9 \\
        \name (ours)                      &  10.7 \\
        \nameFast (ours)                  &  9.0 \\
        \bottomrule
    \end{tblr}
    \vspace{-2mm}
    \caption{\textbf{Performance across all sequences (97 pairs) from \citet{strecha2008benchmark}}. (left) Relative rotation and translation errors (in degrees) for all image pairs. (right) Averaged execution times for computing the relative pose for each method. As can be seen, \nameFast is the fastest among all methods. However, \nameFast is not always tight, resulting in a slight loss of accuracy when compared to the alternatives. On the other hand, our \name is significantly more accurate than \citet{zhao2022nonmin} and is on-par with \citet{garciasalguero2022tighter} (labeled on the left as G-Salg.). Additionally, our \name is, on average, 40\% faster than \cite{garciasalguero2022tighter}.}
    \label{fig:strecha}
    \vspace{-4mm}
\end{figure*}

%% file: main.bbl
\begin{thebibliography}{70}
\providecommand{\natexlab}[1]{#1}
\providecommand{\url}[1]{\texttt{#1}}
\expandafter\ifx\csname urlstyle\endcsname\relax
  \providecommand{\doi}[1]{doi: #1}\else
  \providecommand{\doi}{doi: \begingroup \urlstyle{rm}\Url}\fi

\bibitem[Alcantarilla et~al.(2013)Alcantarilla, Nuevo, and Bartoli]{alcantarilla2013fast}
Pablo Alcantarilla, Jesus Nuevo, and Adrien Bartoli.
\newblock Fast explicit diffusion for accelerated features in nonlinear scale spaces.
\newblock In \emph{BMVC}. British Machine Vision Association, 2013.

\bibitem[Bao et~al.(2011)Bao, Sahinidis, and Tawarmalani]{bao2011semidefinite}
Xiaowei Bao, Nikolaos~V Sahinidis, and Mohit Tawarmalani.
\newblock Semidefinite relaxations for quadratically constrained quadratic programming: A review and comparisons.
\newblock \emph{Mathematical programming}, 129:\penalty0 129--157, 2011.

\bibitem[Beardsley et~al.(1994)Beardsley, Zisserman, and Murray]{beardsley1994navigation}
Paul~A Beardsley, Andrew Zisserman, and David~William Murray.
\newblock Navigation using affine structure from motion.
\newblock In \emph{ECCV}. Springer, 1994.

\bibitem[Boumal(2023)]{boumal2023intromanifolds}
Nicolas Boumal.
\newblock \emph{An introduction to optimization on smooth manifolds}.
\newblock Cambridge University Press, 2023.

\bibitem[Boumal et~al.(2016)Boumal, Voroninski, and Bandeira]{boumal2016staircase}
Nicolas Boumal, Vlad Voroninski, and Afonso Bandeira.
\newblock The non-convex burer-monteiro approach works on smooth semidefinite programs.
\newblock In \emph{NeurIPS}. Curran Associates, Inc., 2016.

\bibitem[Boyd and Vandenberghe(2004)]{boyd2004convex}
Stephen~P Boyd and Lieven Vandenberghe.
\newblock \emph{Convex optimization}.
\newblock Cambridge university press, 2004.

\bibitem[Bradski(2000)]{opencv_library}
G. Bradski.
\newblock {The OpenCV Library}.
\newblock \emph{Dr. Dobb's Journal of Software Tools}, 2000.

\bibitem[Briales and Gonzalez-Jimenez(2017{\natexlab{a}})]{Briales2017registration}
Jesus Briales and Javier Gonzalez-Jimenez.
\newblock Convex global 3d registration with lagrangian duality.
\newblock In \emph{CVPR}, 2017{\natexlab{a}}.

\bibitem[Briales and Gonzalez-Jimenez(2017{\natexlab{b}})]{briales2017cartan}
Jesus Briales and Javier Gonzalez-Jimenez.
\newblock Cartan-sync: Fast and global se(d)-synchronization.
\newblock \emph{IEEE RA-L}, 2\penalty0 (4):\penalty0 2127--2134, 2017{\natexlab{b}}.

\bibitem[Briales et~al.(2018)Briales, Kneip, and Gonzalez-Jimenez]{Briales_2018_CVPR}
Jesus Briales, Laurent Kneip, and Javier Gonzalez-Jimenez.
\newblock {A Certifiably Globally Optimal Solution to the Non-Minimal Relative Pose Problem}.
\newblock In \emph{CVPR}, 2018.

\bibitem[Burer and Monteiro(2003)]{burer2003lowrank}
Samuel Burer and Renato~DC Monteiro.
\newblock A nonlinear programming algorithm for solving semidefinite programs via low-rank factorization.
\newblock \emph{Mathematical programming}, 95\penalty0 (2):\penalty0 329--357, 2003.

\bibitem[Burer and Park(2023)]{burer2023strengthened}
Samuel Burer and Kyungchan Park.
\newblock A strengthened sdp relaxation for quadratic optimization over the stiefel manifold.
\newblock \emph{Journal of optimization theory and applications}, pages 1--20, 2023.

\bibitem[Campos et~al.(2021)Campos, Elvira, Rodríguez, M.~Montiel, and D.~Tardós]{campos2021orb3}
Carlos Campos, Richard Elvira, Juan J.~Gómez Rodríguez, José~M. M.~Montiel, and Juan D.~Tardós.
\newblock Orb-slam3: An accurate open-source library for visual, visual–inertial, and multimap slam.
\newblock \emph{IEEE T-RO}, 37\penalty0 (6):\penalty0 1874--1890, 2021.

\bibitem[Cifuentes et~al.(2022)Cifuentes, Agarwal, Parrilo, and Thomas]{cifuentes2022local}
Diego Cifuentes, Sameer Agarwal, Pablo~A Parrilo, and Rekha~R Thomas.
\newblock On the local stability of semidefinite relaxations.
\newblock \emph{Mathematical Programming}, pages 1--35, 2022.

\bibitem[Dellaert et~al.(2020)Dellaert, Rosen, Wu, Mahony, and Carlone]{dellaert2020shonan}
Frank Dellaert, David~M Rosen, Jing Wu, Robert Mahony, and Luca Carlone.
\newblock Shonan rotation averaging: Global optimality by surfing so (p)\^{} n so (p) n.
\newblock In \emph{ECCV}, pages 292--308. Springer, 2020.

\bibitem[DeTone et~al.(2018)DeTone, Malisiewicz, and Rabinovich]{detone2018superpoint}
Daniel DeTone, Tomasz Malisiewicz, and Andrew Rabinovich.
\newblock Superpoint: Self-supervised interest point detection and description.
\newblock In \emph{CVPRW}, 2018.

\bibitem[Edstedt et~al.(2023{\natexlab{a}})Edstedt, Athanasiadis, Wadenb\"ack, and Felsberg]{Edstedt2023dkm}
Johan Edstedt, Ioannis Athanasiadis, M\r{a}rten Wadenb\"ack, and Michael Felsberg.
\newblock {DKM: Dense Kernelized Feature Matching for Geometry Estimation}.
\newblock In \emph{CVPR}, 2023{\natexlab{a}}.

\bibitem[Edstedt et~al.(2023{\natexlab{b}})Edstedt, Sun, Bökman, Wadenbäck, and Felsberg]{edstedt2023roma}
Johan Edstedt, Qiyu Sun, Georg Bökman, Mårten Wadenbäck, and Michael Felsberg.
\newblock {RoMa: Revisiting Robust Lossses for Dense Feature Matching}.
\newblock \emph{arXiv preprint arXiv:2305.15404}, 2023{\natexlab{b}}.

\bibitem[Fischler and Bolles(1981)]{fischler1981ransac}
Martin~A Fischler and Robert~C Bolles.
\newblock Random sample consensus: a paradigm for model fitting with applications to image analysis and automated cartography.
\newblock \emph{Communications of the ACM}, 24\penalty0 (6):\penalty0 381--395, 1981.

\bibitem[Fujisawa et~al.(2002)Fujisawa, Kojima, Nakata, and Yamashita]{fujisawa2002sdpa62}
Katsuki Fujisawa, Masakazu Kojima, Kazuhide Nakata, and Makoto Yamashita.
\newblock Sdpa (semidefinite programming algorithm) user’s manual—version 6.2. 0.
\newblock \emph{Department of Mathematical and Com-puting Sciences, Tokyo Institute of Technology. Research Reports on Mathematical and Computing Sciences Series B: Operations Research}, 2002.

\bibitem[Garcia-Salguero and Gonzalez-Jimenez(2021)]{garciasalguero2021fast}
Mercedes Garcia-Salguero and Javier Gonzalez-Jimenez.
\newblock Fast and robust certifiable estimation of the relative pose between two calibrated cameras.
\newblock \emph{Journal of Mathematical Imaging and Vision}, 63\penalty0 (8):\penalty0 1036--1056, 2021.

\bibitem[Garcia-Salguero et~al.(2021)Garcia-Salguero, Briales, and Gonzalez-Jimenez]{garciasalguero2021certifiable}
Mercedes Garcia-Salguero, Jesus Briales, and Javier Gonzalez-Jimenez.
\newblock Certifiable relative pose estimation.
\newblock \emph{Image and Vision Computing}, 109:\penalty0 104142, 2021.

\bibitem[Garcia-Salguero et~al.(2022)Garcia-Salguero, Briales, and Gonzalez-Jimenez]{garciasalguero2022tighter}
Mercedes Garcia-Salguero, Jesus Briales, and Javier Gonzalez-Jimenez.
\newblock A tighter relaxation for the relative pose problem between cameras.
\newblock \emph{Journal of Mathematical Imaging and Vision}, 64\penalty0 (5):\penalty0 493--505, 2022.

\bibitem[Giamou(2023)]{giamou2023semidefinite}
Matthew~Peter Giamou.
\newblock \emph{Semidefinite Relaxations for Geometric Problems in Robotics}.
\newblock PhD thesis, University of Toronto (Canada), 2023.

\bibitem[H\"arenstam-Nielsen et~al.(2023)H\"arenstam-Nielsen, Zeller, and Cremers]{Harenstam2023triangulation}
Linus H\"arenstam-Nielsen, Niclas Zeller, and Daniel Cremers.
\newblock Semidefinite relaxations for robust multiview triangulation.
\newblock In \emph{CVPR}, 2023.

\bibitem[Hartley(1997)]{hartley1997eight}
R.I. Hartley.
\newblock In defense of the eight-point algorithm.
\newblock \emph{IEEE TPAMI}, 19\penalty0 (6):\penalty0 580--593, 1997.

\bibitem[Hartley et~al.(2013)Hartley, Trumpf, Dai, and Li]{hartley2013rotation}
Richard Hartley, Jochen Trumpf, Yuchao Dai, and Hongdong Li.
\newblock Rotation averaging.
\newblock \emph{IJCV}, 103:\penalty0 267--305, 2013.

\bibitem[Hartley and Kahl(2009)]{hartley2009globalbnb}
Richard~I Hartley and Fredrik Kahl.
\newblock Global optimization through rotation space search.
\newblock \emph{IJCV}, 82\penalty0 (1):\penalty0 64--79, 2009.

\bibitem[Hartley and Zisserman(2004)]{Hartley2004mvg}
R.~I. Hartley and A. Zisserman.
\newblock \emph{Multiple View Geometry in Computer Vision}.
\newblock Cambridge University Press, ISBN: 0521540518, second edition, 2004.

\bibitem[Helmke et~al.(2007)Helmke, H{\"u}per, Lee, and Moore]{helmke2007essential}
Uwe Helmke, Knut H{\"u}per, Pei~Yean Lee, and John Moore.
\newblock Essential matrix estimation using gauss-newton iterations on a manifold.
\newblock \emph{IJCV}, 74:\penalty0 117--136, 2007.

\bibitem[Higham(1989)]{higham1989matrix}
Nicholas~J Higham.
\newblock Matrix nearness problems and applications.
\newblock \emph{Applications of matrix theory}, 22, 1989.

\bibitem[Horn and Johnson(2012)]{horn2012matrix}
Roger~A Horn and Charles~R Johnson.
\newblock \emph{Matrix Analysis}.
\newblock Cambridge university press, 2012.

\bibitem[K.~C.~Toh and Tütüncü(1999)]{toh1999sdpt3}
M.~J.~Todd K.~C.~Toh and R.~H. Tütüncü.
\newblock {SDPT3 — A Matlab software package for semidefinite programming, Version 1.3}.
\newblock \emph{Optimization Methods and Software}, 11\penalty0 (1-4):\penalty0 545--581, 1999.

\bibitem[Karimian and Tron(2023)]{Karimian2023essential}
Arman Karimian and Roberto Tron.
\newblock Essential matrix estimation using convex relaxations in orthogonal space.
\newblock In \emph{ICCV}, pages 17142--17152, 2023.

\bibitem[Kneip and Furgale(2014)]{kneip2014opengv}
Laurent Kneip and Paul Furgale.
\newblock Opengv: A unified and generalized approach to real-time calibrated geometric vision.
\newblock In \emph{IEEE ICRA}, pages 1--8, 2014.

\bibitem[Kneip and Lynen(2013)]{kneip2013directoptrotbnb}
Laurent Kneip and Simon Lynen.
\newblock Direct optimization of frame-to-frame rotation.
\newblock In \emph{ICCV}, 2013.

\bibitem[Kneip et~al.(2012)Kneip, Siegwart, and Pollefeys]{kneip2012finding}
Laurent Kneip, Roland Siegwart, and Marc Pollefeys.
\newblock Finding the exact rotation between two images independently of the translation.
\newblock In \emph{ECCV}, pages 696--709. Springer, 2012.

\bibitem[Lee and Civera(2019)]{lee2019bmvctriang}
Seong~Hun Lee and Javier Civera.
\newblock Triangulation: Why optimize?
\newblock In \emph{BMVC}, 2019.

\bibitem[Lee and Civera(2020)]{lee2020geometric}
Seong~Hun Lee and Javier Civera.
\newblock Geometric interpretations of the normalized epipolar error.
\newblock \emph{arXiv preprint arXiv:2008.01254}, 2020.

\bibitem[Lee and Civera(2022)]{lee2022hara}
Seong~Hun Lee and Javier Civera.
\newblock Hara: A hierarchical approach for robust rotation averaging.
\newblock In \emph{CVPR}, pages 15777--15786, 2022.

\bibitem[Lindenberger et~al.(2023)Lindenberger, Sarlin, and Pollefeys]{Lindenberger2023lightglue}
Philipp Lindenberger, Paul-Edouard Sarlin, and Marc Pollefeys.
\newblock Lightglue: Local feature matching at light speed.
\newblock In \emph{ICCV}, 2023.

\bibitem[Longuet-Higgins(1981)]{longuet1981computer}
H~Christopher Longuet-Higgins.
\newblock A computer algorithm for reconstructing a scene from two projections.
\newblock \emph{Nature}, 293\penalty0 (5828):\penalty0 133--135, 1981.

\bibitem[Lowe(2004)]{lowe2004distinctive}
David~G Lowe.
\newblock Distinctive image features from scale-invariant keypoints.
\newblock \emph{IJCV}, 60:\penalty0 91--110, 2004.

\bibitem[Moulon et~al.(2016)Moulon, Monasse, Perrot, and Marlet]{moulon2016openmvg}
Pierre Moulon, Pascal Monasse, Romuald Perrot, and Renaud Marlet.
\newblock Open{MVG}: Open multiple view geometry.
\newblock In \emph{International Workshop on Reproducible Research in Pattern Recognition}, pages 60--74. Springer, 2016.

\bibitem[Muhle et~al.(2022)Muhle, Koestler, Demmel, Bernard, and Cremers]{Muhle2022pnec}
Dominik Muhle, Lukas Koestler, Nikolaus Demmel, Florian Bernard, and Daniel Cremers.
\newblock The probabilistic normal epipolar constraint for frame-to-frame rotation optimization under uncertain feature positions.
\newblock In \emph{CVPR}, 2022.

\bibitem[Mur-Artal and Tardós(2017)]{mur2017orb2}
Raúl Mur-Artal and Juan~D. Tardós.
\newblock Orb-slam2: An open-source slam system for monocular, stereo, and rgb-d cameras.
\newblock \emph{IEEE T-RO}, 33\penalty0 (5):\penalty0 1255--1262, 2017.

\bibitem[Mur-Artal et~al.(2015)Mur-Artal, Montiel, and Tardós]{mur2015orb}
Raúl Mur-Artal, J.~M.~M. Montiel, and Juan~D. Tardós.
\newblock Orb-slam: A versatile and accurate monocular slam system.
\newblock \emph{IEEE T-RO}, 31\penalty0 (5):\penalty0 1147--1163, 2015.

\bibitem[Nister(2004)]{nister2004five}
D. Nister.
\newblock An efficient solution to the five-point relative pose problem.
\newblock \emph{IEEE TPAMI}, 26\penalty0 (6):\penalty0 756--770, 2004.

\bibitem[Raguram et~al.(2012)Raguram, Chum, Pollefeys, Matas, and Frahm]{raguram2012usac}
Rahul Raguram, Ondrej Chum, Marc Pollefeys, Jiri Matas, and Jan-Michael Frahm.
\newblock Usac: A universal framework for random sample consensus.
\newblock \emph{IEEE TPAMI}, 35\penalty0 (8):\penalty0 2022--2038, 2012.

\bibitem[Rosen et~al.(2019)Rosen, Carlone, Bandeira, and Leonard]{rosen2019sesync}
David~M Rosen, Luca Carlone, Afonso~S Bandeira, and John~J Leonard.
\newblock Se-sync: A certifiably correct algorithm for synchronization over the special euclidean group.
\newblock \emph{The International Journal of Robotics Research}, 38\penalty0 (2-3):\penalty0 95--125, 2019.

\bibitem[Sanyal et~al.(2011)Sanyal, Sottile, and Sturmfels]{sanyal2011orbitopes}
Raman Sanyal, Frank Sottile, and Bernd Sturmfels.
\newblock Orbitopes.
\newblock \emph{Mathematika}, 57\penalty0 (2):\penalty0 275--314, 2011.

\bibitem[Saunderson et~al.(2015)Saunderson, Parrilo, and Willsky]{saunderson2015semidefinite}
J. Saunderson, P.~A. Parrilo, and A.~S. Willsky.
\newblock Semidefinite descriptions of the convex hull of rotation matrices.
\newblock \emph{SIAM Journal on Optimization}, 25\penalty0 (3):\penalty0 1314--1343, 2015.

\bibitem[Schonberger and Frahm(2016)]{Schonberger2016colmap}
Johannes~L. Schonberger and Jan-Michael Frahm.
\newblock Structure-from-motion revisited.
\newblock In \emph{CVPR}, 2016.

\bibitem[Stewénius et~al.(2006)Stewénius, Engels, and Nistér]{stewenius2006five}
Henrik Stewénius, Christopher Engels, and David Nistér.
\newblock Recent developments on direct relative orientation.
\newblock \emph{ISPRS Journal of Photogrammetry and Remote Sensing}, 60\penalty0 (4):\penalty0 284--294, 2006.

\bibitem[Strang and Algebra(1980)]{strang1980its}
Gilbert Strang and Linear Algebra.
\newblock its applications.
\newblock \emph{Academic Press, New York}, 14:\penalty0 181208, 1980.

\bibitem[Strecha et~al.(2008)Strecha, von Hansen, Van~Gool, Fua, and Thoennessen]{strecha2008benchmark}
C. Strecha, W. von Hansen, L. Van~Gool, P. Fua, and U. Thoennessen.
\newblock On benchmarking camera calibration and multi-view stereo for high resolution imagery.
\newblock In \emph{CVPR}, 2008.

\bibitem[Sturm(1999)]{sturm1999sedumi}
Jos~F. Sturm.
\newblock Using sedumi 1.02, a {MATLAB} toolbox for optimization over symmetric cones.
\newblock \emph{Optimization Methods and Software}, 11\penalty0 (1-4):\penalty0 625--653, 1999.

\bibitem[Sun et~al.(2021)Sun, Shen, Wang, Bao, and Zhou]{sun2021loftr}
Jiaming Sun, Zehong Shen, Yuang Wang, Hujun Bao, and Xiaowei Zhou.
\newblock Loftr: Detector-free local feature matching with transformers.
\newblock In \emph{CVPR}, 2021.

\bibitem[Tron and Daniilidis(2014)]{Tron_2014_CVPR}
Roberto Tron and Kostas Daniilidis.
\newblock On the quotient representation for the essential manifold.
\newblock In \emph{CVPR}, 2014.

\bibitem[Tron and Daniilidis(2017)]{tron2017essential}
Roberto Tron and Kostas Daniilidis.
\newblock The space of essential matrices as a riemannian quotient manifold.
\newblock \emph{SIAM Journal on Imaging Sciences}, 10\penalty0 (3):\penalty0 1416--1445, 2017.

\bibitem[Truong et~al.(2023)Truong, Danelljan, Timofte, and Van~Gool]{truong2023pdcnetpami}
Prune Truong, Martin Danelljan, Radu Timofte, and Luc Van~Gool.
\newblock {PDC-Net+: Enhanced Probabilistic Dense Correspondence Network}.
\newblock \emph{IEEE TPAMI}, 45\penalty0 (8):\penalty0 10247--10266, 2023.

\bibitem[Van~Loan(2000)]{van2000ubiquitous}
Charles~F Van~Loan.
\newblock The ubiquitous kronecker product.
\newblock \emph{Journal of computational and applied mathematics}, 123\penalty0 (1-2):\penalty0 85--100, 2000.

\bibitem[Vandenberghe and Boyd(1996)]{vandenberghe1996semidefinite}
Lieven Vandenberghe and Stephen Boyd.
\newblock Semidefinite programming.
\newblock \emph{SIAM review}, 38\penalty0 (1):\penalty0 49--95, 1996.

\bibitem[Yamashita et~al.(2010)Yamashita, Fujisawa, Nakata, Nakata, Fukuda, Kobayashi, and Goto]{yamashita2010sdpa7}
Makoto Yamashita, Katsuki Fujisawa, Kazuhide Nakata, Maho Nakata, Mituhiro Fukuda, Kazuhiro Kobayashi, and Kazushige Goto.
\newblock A high-performance software package for semidefinite programs: Sdpa 7.
\newblock Technical report, 2010.

\bibitem[Yang and Carlone(2019)]{Yang2019quasarwahba}
Heng Yang and Luca Carlone.
\newblock A quaternion-based certifiably optimal solution to the wahba problem with outliers.
\newblock In \emph{ICCV}, 2019.

\bibitem[Yang et~al.(2020)Yang, Antonante, Tzoumas, and Carlone]{yang2020gnc}
Heng Yang, Pasquale Antonante, Vasileios Tzoumas, and Luca Carlone.
\newblock Graduated non-convexity for robust spatial perception: From non-minimal solvers to global outlier rejection.
\newblock \emph{IEEE RA-L}, 5\penalty0 (2):\penalty0 1127--1134, 2020.

\bibitem[Yang et~al.(2021)Yang, Shi, and Carlone]{yang2021teaserplusplus}
Heng Yang, Jingnan Shi, and Luca Carlone.
\newblock Teaser: Fast and certifiable point cloud registration.
\newblock \emph{IEEE T-RO}, 37\penalty0 (2):\penalty0 314--333, 2021.

\bibitem[Yang et~al.(2014)Yang, Li, and Jia]{yang2014optimalbnb}
Jiaolong Yang, Hongdong Li, and Yunde Jia.
\newblock Optimal essential matrix estimation via inlier-set maximization.
\newblock In \emph{ECCV}. Springer, 2014.

\bibitem[Zhao(2022)]{zhao2022nonmin}
Ji Zhao.
\newblock An efficient solution to non-minimal case essential matrix estimation.
\newblock \emph{IEEE TPAMI}, 44\penalty0 (4):\penalty0 1777--1792, 2022.

\bibitem[Zhao et~al.(2020)Zhao, Xu, and Kneip]{Zhao2020generalizedessmat}
Ji Zhao, Wanting Xu, and Laurent Kneip.
\newblock A certifiably globally optimal solution to generalized essential matrix estimation.
\newblock In \emph{CVPR}, 2020.

\end{thebibliography}
